\documentclass[journal]{IEEEtran}
\usepackage{amsmath,amssymb,amsfonts}
\usepackage{cite}
\usepackage{algorithm}
\usepackage{algorithmic}
\usepackage{graphicx}
\usepackage{caption}
\usepackage{lipsum}
\usepackage{enumerate}
\usepackage{amsthm}
\usepackage{color}
\usepackage{booktabs}
\newcommand{\Exp}{\mathop{\mathbb E}\displaylimits}

\newtheorem{lemma}{Lemma}
\newtheorem{theorem}{Theorem}

\ifCLASSOPTIONcompsoc
  \usepackage[caption=false,font=normalsize,labelfont=sf,textfont=sf]{subfig}
\else
  \usepackage[caption=false,font=footnotesize]{subfig}
\fi

\usepackage{tikz}
\usepackage{textcomp}
\usepackage{hyperref}
\newcommand\copyrighttext{%
  \footnotesize \textcopyright 2021 IEEE. Personal use of this material is permitted.
  Permission from IEEE must be obtained for all other uses, in any current or future
  media, including reprinting/republishing this material for advertising or promotional
  purposes, creating new collective works, for resale or redistribution to servers or
  lists, or reuse of any copyrighted component of this work in other works.}
\newcommand\copyrightnotice{%
\begin{tikzpicture}[remember picture,overlay]
\node[anchor=south,yshift=10pt] at (current page.south) {\fbox{\parbox{\dimexpr\textwidth-\fboxsep-\fboxrule\relax}{\copyrighttext}}};
\end{tikzpicture}%
}

\ifCLASSINFOpdf
\else
\fi
\begin{document}

%
% paper title
% Titles are generally capitalized except for words such as a, an, and, as,
% at, but, by, for, in, nor, of, on, or, the, to and up, which are usually
% not capitalized unless they are the first or last word of the title.
% Linebreaks \\ can be used within to get better formatting as desired.
% Do not put math or special symbols in the title.
\title{Distributional Soft Actor-Critic: Off-Policy Reinforcement Learning for Addressing Value Estimation Errors}
%
%
% author names and IEEE memberships
% note positions of commas and nonbreaking spaces ( ~ ) LaTeX will not break
% a structure at a ~ so this keeps an author's name from being broken across
% two lines.
% use \thanks{} to gain access to the first footnote area
% a separate \thanks must be used for each paragraph as LaTeX2e's \thanks
% was not built to handle multiple paragraphs
%

\author{Jingliang Duan, Yang Guan, Shengbo Eben Li*, Yangang Ren, Qi Sun, and~Bo Cheng% <-this % stops a space
\thanks{This study is supported by Beijing NSF with JQ18010, and NSF China with 51575293, and U20A20334. Special thanks should be given to TOYOTA for funding this study. Jingliang Duan and Yang Guan contributed equally to this work. All correspondences should be sent to S. Li with email: lisb04@gmail.com. }% <-this % stops a space
\thanks{J. Duan, Y. Guan, S. Li, Y. Ren, Q. Sun, and B. Cheng are with State Key Lab of Automotive Safety and Energy, School of Vehicle and Mobility, Tsinghua University, Beijing, 100084, China. They are also with Center for Intelligent Connected Vehicles and Transportation, Tsinghua University. {\tt\small Email: duanjl15@163.com; (guany17, ryg18)@mails.tsinghua.edu.cn; (lishbo, qisun, chengbo)@tsinghua.edu.cn}.
}% <-this % stops a space
}

\maketitle

\copyrightnotice

%As a general rule, do not put math, special symbols or citations
% in the abstract or keywords.
\begin{abstract}
In reinforcement learning (RL), function approximation errors are known to easily lead to the Q-value overestimations, thus greatly reducing policy performance. This paper presents a distributional soft actor-critic (DSAC) algorithm, which is an off-policy RL method for continuous control setting, to improve the policy performance by mitigating Q-value overestimations. We first discover in theory that learning a distribution function of state-action returns can effectively mitigate Q-value overestimations because it is capable of adaptively adjusting the update stepsize of the Q-value function. Then, a distributional soft policy iteration (DSPI) framework is developed by embedding the return distribution function into maximum entropy RL. Finally, we present a deep off-policy actor-critic variant of DSPI, called DSAC, which directly learns a continuous return distribution by keeping the variance of the state-action returns within a reasonable range to 
address exploding and vanishing gradient problems. \textcolor{black}{We evaluate DSAC on the suite of MuJoCo continuous control tasks, achieving the state-of-the-art performance.}
\end{abstract}

% Note that keywords are not normally used for peerreview papers.

% For peer review papers, you can put extra information on the cover
% page as needed:
% \ifCLASSOPTIONpeerreview
% \begin{center} \bfseries EDICS Category: 3-BBND \end{center}
% \fi
%
% For peerreview papers, this IEEEtran command inserts a page break and
% creates the second title. It will be ignored for other modes.
\IEEEpeerreviewmaketitle

\section{Introduction}
% The very first letter is a 2 line initial drop letter followed
% by the rest of the first word in caps.
% 
% form to use if the first word consists of a single letter:
% \IEEEPARstart{A}{demo} file is ....
% 
% form to use if you need the single drop letter followed by
% normal text (unknown if ever used by the IEEE):
% \IEEEPARstart{A}{}demo file is ....
% 
% Some journals put the first two words in caps:
% \IEEEPARstart{T}{his demo} file is ....
% 
% Here we have the typical use of a "T" for an initial drop letter
% and "HIS" in caps to complete the first word.
\IEEEPARstart{D}{eep} neural networks (NNs) provide rich representations that can enable reinforcement learning (RL) algorithms to master a variety of challenging domains, from games to robotic control \cite{mnih2015DQN,silver2016mastering,mnih2016A3C,silver2017mastering,duan2020hierarchical}. However, most RL algorithms tend to learn unrealistically high state-action values (i.e., Q-values), known as overestimations, thereby resulting in suboptimal policies.

The overestimations of RL were first found in the Q-learning algorithm \cite{watkins1989Q-learning}, which is the prototype of most existing value-based RL algorithms \cite{sutton2018reinforcement}. For this algorithm, van Hasselt \emph{et al.} (2016) demonstrated that any kind of estimation errors can induce an upward bias, irrespective of whether these errors are caused by system noise, function approximation, or any other sources \cite{van2016double_DQN}. The overestimation bias is firstly induced by the max operator over all noisy Q-estimates of the same state, which tends to prefer overestimated to underestimated Q-values \cite{thrun1993issues,lee2013bias,lee2019bias}. This overestimation bias will be further propagated and exaggerated through the temporal difference learning \cite{sutton2018reinforcement}, wherein the Q-estimate of a state is updated using the Q-estimate of its subsequent state. Deep RL algorithms, such as Deep Q-Networks (DQN) \cite{mnih2015DQN}, employ a deep NN to estimate the Q-value. Although the deep NN can provide rich representations with the potential for low asymptotic approximation errors, overestimations still exist, even in deterministic environments \cite{van2016double_DQN,Fujimoto2018TD3}. Fujimoto \emph{et al.} (2018) showed that the overestimation problem also persists in actor-critic RL \cite{Fujimoto2018TD3}, such as Deterministic Policy Gradient (DPG) and Deep DPG (DDPG) \cite{silver2014DPG,lillicrap2015DDPG}. In practice, inaccurate estimation exists in almost all RL algorithms because, on the one hand, any algorithm will introduce some estimation biases and variances, simply due to the true Q-values are initially unknown \cite{sutton2018reinforcement}. On the other hand, function approximation errors are usually unavoidable. This is particularly problematic because inaccurate estimation can cause arbitrarily suboptimal actions to be overestimated, resulting in a suboptimal policy.

To reduce overestimations in standard Q-learning, Double Q-learning \cite{hasselt2010double_Q} was developed to decouple the max operation into action selection and evaluation. To update one of these two Q-networks, one Q-network is used to determine the greedy policy, while another Q-network is used to determine its value, resulting in unbiased estimates. Double DQN \cite{van2016double_DQN}, a deep variant of Double Q-learning, deals with the overestimation problem of DQN, in which the target Q-network of DQN provides a natural candidate for the second Q-network. %The target Q network of DQN provides a natural candidate for the second Q network, which is used to make estimates of the actions selected using the online Q network. 
However, these two methods can only handle discrete action spaces. Fujimoto \emph{et al.} (2018) developed actor-critic variants of the standard Double DQN and Double Q-learning for continuous control, by making action selections using the policy optimized with respect to the corresponding Q-estimate \cite{Fujimoto2018TD3}. However, the actor-critic Double DQN suffers from similar overestimations as DDPG, because the online and target Q-estimates are too similar to provide an independent estimation. While actor-critic Double Q-learning is more effective, it introduces additional Q and policy networks at the cost of increasing the computation time for each iteration. Finally, Fujimoto \emph{et al.} (2018) proposed Clipped Double Q-learning by taking the minimum value between the two Q-estimates \cite{Fujimoto2018TD3}, which is used in Twin Delayed Deep Deterministic policy gradient (TD3) and Soft Actor-Critic (SAC) \cite{Haarnoja2018SAC,Haarnoja2018ASAC}. %They also extended this method to incorporate a number of modifications, such as slow-updating target NNs and delayed policy updates, to reduce the estimation variance, and proposed the Twin Delayed Deep Deterministic policy gradient (TD3) algorithm based on DDPG.  
However, this method may introduce a considerable underestimation bias and still requires an additional Q-network.

{\color{black}{In this paper, we propose a new RL algorithm, called distributional soft actor-critic (DSAC), to improve policy performance by mitigating Q-value overestimations. The contributions and novelty of this paper are summarized as follows:
\begin{enumerate}
    \item A distributional soft policy iteration (DSPI) framework is developed by embedding the return distribution function in maximum entropy RL to learn a continuous distribution of state-action returns (also called return distribution). The impact of the return distribution learning on the accuracy of Q-value estimation was barely discussed in existing distributional RL algorithms, such as \cite{bellemare2017C51,dabney2018quantileregre,davney2018quantilenet,rowland2018distributional_ana,lyle2019comparative_dstributional,barth-maron2018D4PG}.  In this paper, we first found that the Q-value overestimations can be mitigated by learning a distribution function of state-action returns. This is because that compared with most RL algorithms that directly learn the expectation of state-action returns (i.e., Q-value) \cite{mnih2015DQN, lillicrap2015DDPG,van2016double_DQN, mnih2016A3C, Fujimoto2018TD3,Haarnoja2018SAC}, the return distribution learning is capable of adaptively adjusting the update stepsize of Q-values. 
    \item  Based on the developed DSPI framework, we propose the DSAC algorithm by replacing the clipped double Q-learning of SAC \cite{Haarnoja2018SAC,Haarnoja2018ASAC} with the return distribution learning. In comparison with RL algorithms that use double value networks to mitigate overestimations \cite{hasselt2010double_Q,van2016double_DQN,Fujimoto2018TD3,Haarnoja2018SAC,Haarnoja2018ASAC}, DSAC improves the Q-value estimation accuracy by only employing a single return distribution network, which also leads to higher time efficiency.
    \item Different from existing distributional RL algorithms that learn a discrete return distribution \cite{bellemare2017C51,dabney2018quantileregre,davney2018quantilenet,rowland2018distributional_ana,lyle2019comparative_dstributional,barth-maron2018D4PG}, the proposed DSAC is capable of learning a continuous return distribution by keeping the variance of the state-action returns within a reasonable range to address exploding and vanishing gradient problems. Therefore, DSAC relaxes the need for human-designed discrete ranges and intervals. Besides, compared with most distributional RL algorithms that can only handle discrete and low-dimensional action spaces \cite{bellemare2017C51,dabney2018quantileregre,davney2018quantilenet,rowland2018distributional_ana,lyle2019comparative_dstributional}, DSAC is applicable to continuous control settings by optimizing an independent stochastic policy network.
    \item Experiments on MuJoCo benchmarks demonstrate that the proposed DSAC algorithm outperforms or matches all baselines across all benchmark tasks in terms of the final performance.
\end{enumerate}
}}

The paper is organized as follows. In Section \ref{sec.related_work}, we introduce the related works. Section \ref{sec:preliminaries} describes some preliminaries of RL and develops a DSPI framework. In Section \ref{sec:analysis}, we analyze the role of the distributional return function in solving overestimations. Section \ref{sec:DSAC} presents the DSAC algorithm and PABAL architecture. In Section \ref{sec:experiments}, we present experimental results that show the efficacy of DSAC. Section \ref{sec:conclusion} concludes this paper.

\section{Related Work}
\label{sec.related_work}
Over the last decade, numerous deep RL algorithms have appeared \cite{mnih2015DQN,lillicrap2015DDPG,schulman2015TRPO,mnih2016A3C,schulman2017PPO,heess2017DPPO,Fujimoto2018TD3,barth-maron2018D4PG,Haarnoja2018SAC}. This paper aims to propose a new RL algorithm to mitigate Q-value overestimations by learning a distribution of state-action returns, thereby improving policy performance. We also incorporate the off-policy formulation to improve sample efficiency, and the maximum entropy framework based on the stochastic policy to encourage exploration. Besides, our algorithm mainly focuses on continuous control setting.  With reference to algorithms such as DDPG \cite{lillicrap2015DDPG}, the off-policy learning and continuous control can be easily enabled by learning separate Q and policy networks in an actor-critic architecture. Therefore, we mainly review prior works on the maximum entropy framework and distributional RL in this section.

Maximum entropy RL favors stochastic policies by augmenting the optimization objective with the expected policy entropy. While many prior RL algorithms consider the policy entropy, they only use it as a regularizer \cite{schulman2015TRPO,mnih2016A3C,schulman2017PPO}. Recently, several papers have noted the connection between Q-learning and policy gradient methods in the setting of the maximum entropy framework \cite{ODonoghue2016PGQL,schulman2017PG_Soft-Q,nachum2017bridging}. Early maximum entropy RL algorithms usually only consider the policy entropy of current states \cite{sallans2004Boltzmann_explore,ODonoghue2016PGQL,Fox2016G-learning}. Unlike them, soft Q-learning directly augments the reward with an entropy term, such that the optimal policy aims to reach states where they will have high policy entropy in the future \cite{Haarnoja2017Soft-Q}. Haarnoja \emph{et al.} (2018) further developed an off-policy actor-critic variant of the Soft Q-learning for large continuous domains, called SAC \cite{Haarnoja2018SAC,Haarnoja2018ASAC}. 
%Haarnoja \emph{et al.} (2018) later devised a gradient-based method for SAC that can automatically learn the optimal temperature of entropy term during training \cite{Haarnoja2018ASAC}. 
In this paper, we build on the work of \cite{Haarnoja2018SAC,Haarnoja2018ASAC} for implementing the maximum entropy framework.

The distributional RL, in which one models the distribution over returns, whose expectation is the value function, was recently introduced by Bellemare \emph{et al.} \cite{bellemare2017C51}. They proposed a distributional RL algorithm, called C51, which achieved great performance improvements on many Atari 2600 benchmarks. Since then, many distributional RL algorithms and their inherent analyses have appeared in literature  \cite{dabney2018quantileregre,davney2018quantilenet,rowland2018distributional_ana,lyle2019comparative_dstributional}. Like DQN, these works can only handle discrete and low-dimensional action spaces, as they select actions according to their Q-networks. Barth-Maron \emph{et al.} (2018) combined the distributional return function within an actor-critic framework for policy learning in continuous control setting domains, and proposed the Distributed Distributional Deep Deterministic Policy Gradient algorithm (D4PG) \cite{barth-maron2018D4PG}. Inspired by these distributional RL researches, Dabney \emph{et al.} (2020) found that the brain represents possible future rewards not as a single mean, but instead as a probability distribution through mouse experiments \cite{dabney2020dopamine-basedRL}. Existing distributional RL algorithms usually learn a discrete return distribution because it is computationally friendly. However, this poses a problem: we need to divide the return distribution into multiple discrete intervals in advance. This is inconvenient because different tasks usually require different division numbers and intervals. In addition, the role of distributional return function in solving overestimations was barely discussed before.

\section{Preliminaries and Distributional Soft Policy Iteration}
\label{sec:preliminaries}

\textcolor{black}{In this section, we first describe the notations and introduce the concept of maximum entropy RL. Then the distributional soft policy iteration (DSPI) framework is developed.}
\subsection{Notation}
\label{sec:notation}
We consider the standard reinforcement learning (RL) setting wherein an agent interacts with an environment $\mathcal{E}$ in discrete time. This environment can be modeled as a Markov Decision Process, defined by the tuple $(\mathcal{S},\mathcal{A},\mathcal{R},p)$. The state space $\mathcal{S}$ and action space $\mathcal{A}$ are assumed to be continuous,  $R(r_t|s_t,a_t):\mathcal{S}\times\mathcal{A}\rightarrow\mathcal{P}(r_t)$ is a stochastic reward function mapping a state-action pair $(s_t,a_t)$ to a distribution over a set of bounded rewards, and the unknown state transition probability $p(s_{t+1}|s_t,a_t):\mathcal{S}\times\mathcal{A}\rightarrow \mathcal{P}(s_{t+1})$ maps a given $(s_t, a_t)$ to the probability distribution over $s_{t+1}$. For the sake of simplicity, the current and next state-action pairs are also denoted as $(s,a)$ and $(s',a')$, respectively.

At each time step $t$, the agent receives a state $s_t\in \mathcal{S}$ and selects an action $a_t\in \mathcal{A}$. In return, the agent receives the next state $s_{t+1}\in \mathcal{S}$ and a scalar reward $r_t\sim R(s_t,a_t)$. The process continues until the agent reaches a terminal state after which the process restarts. The agent's behavior is defined by a stochastic policy $\pi(a_t|s_t): \mathcal{S}\rightarrow\mathcal{P}(a_t)$, which maps a given state to a probability distribution over actions. We will use $\rho_{\pi}(s)$ and $\rho_{\pi}(s,a)$ to denote the state and state-action distribution induced by policy $\pi$.

\subsection{Maximum Entropy RL}
\label{sec:max_entropy}
The goal in standard RL is to learn a policy which maximizes the expected future accumulated return $\Exp_{(s_{i\ge t},a_{i\ge t})\sim \rho_{\pi}, r_{i\ge t}\sim R(\cdot|s_i,a_i)}[\sum^{\infty}_{i=t}\gamma^{i-t} r_i]$, where $\gamma \in [0,1)$ is the discount factor. In this paper, we consider a more general entropy-augmented objective \cite{Haarnoja2017Soft-Q,Haarnoja2018SAC,Haarnoja2018ASAC}, which augments the reward with a policy entropy term $\mathcal{H}$,
\begin{equation}
\label{eq.policy_objective}
J_{\pi} = \Exp_{\substack{(s_{i \ge t},a_{i \ge t})\sim \rho_{\pi},\\r_{i \ge t}\sim R(\cdot|s_i,a_i)}}\Big[\sum^{\infty}_{i=t}\gamma^{i-t} [r_i+\alpha\mathcal{H}(\pi(\cdot|s_i))]\Big],
\end{equation}
{\color{black}{where 
\begin{equation}
\nonumber
\begin{aligned}
\mathcal{H}(\pi(\cdot|s))&= -\int_{a\in \mathcal{A}}\pi(a|s)\log\pi(a|s){\rm d}a\\
&=\Exp_{a\sim\pi(\cdot|s)}\big[-\log\pi(a|s)\big].
\end{aligned}
\end{equation}}}This objective improves the exploration efficiency of the policy by maximizing both the expected future return and policy entropy. The temperature parameter $\alpha$ determines the relative importance of the entropy term against the reward. Maximum entropy RL gradually approaches the conventional RL as $\alpha \rightarrow 0$.

We use $G_t=\sum^{\infty}_{i=t}\gamma^{i-t} [r_i-\alpha \log\pi(a_i|s_i)]$ to denote the entropy-augmented accumulated return from $s_t$, also called soft return. The soft Q-value of policy $\pi$ is defined as 
\begin{equation}
\label{eq.Q_definition}
Q^{\pi}(s_t,a_t)=\Exp_{r\sim R(\cdot|s_t,a_t)}[r]+\gamma\Exp_{\substack{(s_{i>t},a_{i>t})\sim \rho_{\pi},\\ r_{i>t}\sim R(\cdot|s_i,a_i)}}[G_{t+1}],
\end{equation}
which describes the expected soft return for selecting $a_t$ in state $s_t$ and thereafter following policy $\pi$. 

The optimal maximum entropy policy is learned by a maximum entropy variant of the policy iteration method, which alternates between soft policy evaluation and soft policy improvement, called soft policy iteration. In the soft policy evaluation process, given a policy $\pi$, the soft Q-value can be learned by repeatedly applying a soft Bellman operator $\mathcal{T}^{\pi}$ under policy $\pi$ given by
\begin{equation}
\label{eq.soft_bellman}
\begin{aligned}
\mathcal{T}^{\pi}Q^{\pi}(s,&a)=\mathbb{E}_{r\sim R(\cdot|s,a)}[r]+\\
&\gamma \mathbb{E}_{s'\sim p,a'\sim \pi}[Q^{\pi}(s',a')-\alpha \log\pi(a'|s')\big].
\end{aligned}
\end{equation}

The goal of the soft policy improvement process is to find a new policy $\pi_{\rm{new}}$ that is better than the current policy $\pi_{\rm{old}}$, such that $J_{\pi_{\rm{new}}}\ge J_{\pi_{\rm{old}}}$.
Hence, we can update the policy directly by maximizing the entropy-augmented objective in \eqref{eq.policy_objective} in terms of the soft Q-value,
\begin{equation}
\label{eq.policy_imp}
\begin{aligned}
\pi_{\rm{new}}&=\arg\max_{\pi}J_{\pi}\\
&=\arg\max_{\pi} \Exp_{s\sim \rho_{\pi},a\sim \pi}\big[Q^{\pi_{\rm{old}}}(s,a)-\alpha \log\pi(a|s)\big].
\end{aligned}
\end{equation}

The convergence and optimality of soft policy iteration have been verified in \cite{Haarnoja2017Soft-Q,Haarnoja2018SAC,Haarnoja2018ASAC, schulman2017PG_Soft-Q}.

\subsection{Distributional Soft Policy Iteration}
\label{sec:DSPI}
\textcolor{black}{Next, we develop the distributional soft policy iteration (DSPI) framework by extending the maximum entropy RL into a distributional learning version. Firstly, we define the soft state-action return of policy $\pi$ from a state-action pair $(s_t,a_t)$ as} 
\begin{equation}
\nonumber
Z^{\pi}(s_t,a_t)=r_t+\gamma G_{t+1}\big|{}_{(s_{i>t},a_{i>t})\sim \rho_{\pi}, r_{i\ge t}\sim R(\cdot|s_i,a_i)},
\end{equation}
which is usually a random variable due to the randomness in the state transition $p$, reward function $R$ and policy $\pi$. From \eqref{eq.Q_definition}, it is clear that 
\begin{equation}
\label{eq.Q_equal_exp_Z}
Q^{\pi}(s,a)=\mathbb{E}[Z^{\pi}(s,a)].
\end{equation}
Instead of just considering the expected state-action return $Q^{\pi}(s,a)$, one can choose to directly model the distribution of the soft returns $Z^{\pi}(s,a)$. We define $\mathcal{Z}^{\pi}(Z^{\pi}(s,a)|s,a): \mathcal{S}\times\mathcal{A}\rightarrow \mathcal{P}(Z^{\pi}(s,a))$ as a mapping from $(s,a)$ to a distribution over soft state-action returns, and call it the \emph{soft state-action return distribution} or distributional value function. The distributional variant of the Bellman operator in 
the maximum entropy framework can be derived as 
\begin{equation}\label{eq.soft_distri_bellman}
\mathcal{T}^{\pi}_{\mathcal{D}}Z^{\pi}(s,a) \overset{D}{=}r+\gamma( Z^{\pi}(s',a')-\alpha \log\pi(a'|s')),
\end{equation}
where $r\sim R(\cdot|s,a),s'\sim p,a'\sim\pi$, and $A \overset{D}{=} B$ denotes that two random variables $A$ and $B$ have equal probability laws. The distributional variant of policy iteration has been proved to converge to the optimal return distribution and policy uniformly in \cite{bellemare2017C51}. We can further prove that DSPI which alternates between \eqref{eq.soft_distri_bellman} and \eqref{eq.policy_imp} also leads to policy improvement with respect to the maximum entropy objective \eqref{eq.policy_objective}. Details are provided in Appendix \ref{appen.proof}.

Suppose $\mathcal{T}^{\pi}_{\mathcal{D}}Z(s,a) \sim \mathcal{T}^{\pi}_{\mathcal{D}}\mathcal{Z}(\cdot|s,a)$, where $\mathcal{T}^{\pi}_{\mathcal{D}}\mathcal{Z}(\cdot|s,a)$ denotes the distribution of $\mathcal{T}^{\pi}_{\mathcal{D}}Z(s,a)$. To implement \eqref{eq.soft_distri_bellman}, we can directly update the soft return distribution by 
\begin{equation}
\label{eq.distributional_Bellman}
\mathcal{Z}_{\rm{new}} =  \arg\min_{\mathcal{Z}}\mathop{\mathbb{E}}_{(s,a)\sim\rho_{\pi}}\big[d(\mathcal{T}^{\pi}_{\mathcal{D}}\mathcal{Z}_{\rm{old}}(\cdot|s,a),\mathcal{Z}(\cdot|s,a))\big],
\end{equation}
where $d$ is some metric to measure the distance between two distributions. For calculation convenience, many practical distributional RL algorithms employ Kullback-Leibler (KL) divergence, denoted as $D_{\rm{KL}}$, as the metric \cite{bellemare2017C51,barth-maron2018D4PG}.

\section{Overestimation Bias}
\label{sec:analysis}
This section mainly focuses on the impact of the state-action return distribution learning on reducing overestimation. Therefore, the entropy coefficient $\alpha$ is assumed to be $0$ here. \textcolor{black}{Previous studies analyzed the Q-value estimation bias of Q-learning in tabular cases \cite{watkins1989Q-learning,hasselt2010double_Q}. In section \ref{sec:overestimate_in_Q-learning}, we derive the analytical expression of Q-value estimation bias from the perspective of function approximation. Then, Section \ref{sec:overesimation_in_distributional} analyzes the Q-estimate bias of the return distribution learning and reveals its mechanism to mitigate overestimations.}

\subsection{Overestimation in Q-learning}
\label{sec:overestimate_in_Q-learning}
In Q-learning with discrete actions, suppose the Q-value is approximated by a Q-function $Q_{\theta}(s,a)$ with parameters $\theta$. Defining the greedy target $y = \mathbb{E}[r] + \gamma\mathbb{E}_{s'}[\max_{a'}Q_{\theta}(s',a')]$, the Q-estimate $Q_{\theta}(s,a)$ can be updated by minimizing the loss $(y-Q_{\theta}(s,a))^2/2$ using gradient descent methods, i.e., 
\begin{equation}
\label{eq.approximate_theta}
\theta_{\rm{new}}=\theta+\beta (y-Q_{\theta}(s,a)) \nabla_{\theta}Q_{\theta}(s,a), 
\end{equation}
where $\beta$ is the learning rate. However, in practical applications, the Q-estimate $Q_{\theta}(s,a)$ usually contains random errors, which may be caused by system noises and function approximation. Denoting the current true Q-value as $\tilde{Q}$, we assume
\begin{equation}
\label{eq.approximate_error}
Q_{\theta}(s,a)=\tilde{Q}(s,a)+\epsilon_Q,
\end{equation}
where the random error $\epsilon_Q$ has zero mean and is independent of $(s,a)$ and $\theta$. To distinguish the random error of $Q_{\theta}(s,a)$ and $Q_{\theta}(s',a')$, the random error of $Q_{\theta}(s',a')$ is denoted as $\epsilon_Q'$. Clearly, $\epsilon_Q'$ may cause inaccuracy on the right-hand side of \eqref{eq.approximate_theta}. Let $\theta_{\rm{true}}$ represent the post-update parameters obtained based on true target $\tilde{y}$, that is,
\begin{equation}
\nonumber
\theta_{\rm{true}}=\theta+\beta (\tilde{y}-Q_{\theta}(s,a)) \nabla_{\theta} Q_{\theta}(s,a),
\end{equation}
where $\tilde{y}=\mathbb{E}[r] + \gamma\mathbb{E}_{s'}[\max_{a'}\tilde{Q}(s',a')]$.

Supposing $\beta$ is sufficiently small, the post-update Q-function can be well-approximated by linearizing around $\theta$ using Taylor's expansion:
\begin{equation}
\nonumber
Q_{\theta_{\rm{true}}}(s,a)\approx Q_{\theta}(s,a)+\beta (\tilde{y}-Q_{\theta}(s,a)) \|\nabla_{\theta} Q_{\theta}(s,a)\|^2_2,
\end{equation}
\begin{equation}
\nonumber
Q_{\theta_{\rm{new}}}(s,a)\approx Q_{\theta}(s,a)+\beta (y-Q_{\theta}(s,a)) \|\nabla_{\theta} Q_{\theta}(s,a)\|^2_2.
\end{equation}

Then, in expectation, the estimate bias of post-update Q-estimate $Q_{\theta_{\rm{new}}}(s,a)$ is 
\begin{equation}
\nonumber
\begin{aligned}
\Delta(s,a)&= \mathbb{E}_{\epsilon_Q'}[Q_{\theta_{\rm{new}}}(s,a)-Q_{\theta_{\rm{true}}}(s,a)]\\
&\approx\beta\big(\mathbb{E}_{\epsilon_Q'}[y]-\tilde{y}\big)\|\nabla_{\theta} Q_{\theta}(s,a)\|^2_2\\
&=\beta\gamma\big(\mathbb{E}_{\epsilon_Q'}\big[\mathbb{E}_{s'}[\max_{a'}Q(s',a')]\big]-\\
&\qquad \qquad \qquad \mathbb{E}_{s'}[\max_{a'}\tilde{Q}(s',a')]\big)\|\nabla_{\theta} Q_{\theta}(s,a)\|^2_2.
\end{aligned}
\end{equation}
Defining
\begin{equation}
\label{eq.delta_definition}
\begin{aligned}
\delta&=\mathbb{E}_{\epsilon_Q'}\big[\mathbb{E}_{s'}[\max_{a'}Q(s',a')]\big]-\mathbb{E}_{s'}[\max_{a'}\tilde{Q}(s',a')]\\
&=\mathbb{E}_{s'}\big[\mathbb{E}_{\epsilon_Q'}[\max_{a'}Q_{\theta}(s',a')]-\max_{a'}\tilde{Q}(s',a')\big]\\
&=\mathbb{E}_{s'}\big[\mathbb{E}_{\epsilon_Q'}[\max_{a'}(\tilde{Q}_{\theta}(s',a')+\epsilon_Q']-\max_{a'}\tilde{Q}(s',a')\big],
\end{aligned}
\end{equation}
$\Delta(s,a)$ can be rewritten as: 
\begin{equation}
\nonumber
\Delta(s,a)\approx\beta \gamma \delta \|\nabla_{\theta} Q_{\theta}(s,a)\|_2^2.
\end{equation}

Although $\epsilon_Q'$ is independent of $(s',a')$, it cannot be extracted from the max operator of $\max_{a'}(\tilde{Q}(s',a')+\epsilon_Q')$. This is because for each $(s', a')$, $\epsilon_Q'$ is a random variable rather than a fixed value. In fact,
it has been verified by previous researches that $\mathbb{E}_{\epsilon_Q'}[\max_{a'}(\tilde{Q}(s',a')+\epsilon_Q')] - \max_{a'}\tilde{Q}(s',a')\ge0$ \cite{thrun1993issues,hasselt2010double_Q}. Therefore, it is clear that
\begin{equation}
\nonumber
\Delta(s,a)\ge0,
\end{equation}
which indicates that $\Delta(s,a)$ is an upward bias. In fact, any kind of estimation errors can induce an upward bias due to the max operator. Although it is reasonable to expect a small upward bias caused by single update, these overestimation errors can be further exaggerated through temporal difference (TD) learning, which may result in large overestimation bias and suboptimal policy updates. 

\subsection{Return Distribution for Reducing Overestimation}
\label{sec:overesimation_in_distributional}

Before discussing the distributional version of Q-learning, we first assume that the random returns $Z(s,a)$ obey a Gaussian distribution $\mathcal{Z}(\cdot|s,a)$. Suppose the mean (i.e., Q-value) and standard deviation of the Gaussian distribution are approximated by two independent functions $Q_{\theta}(s,a)$ and $\sigma_{\psi}(s,a)$, with parameters $\theta$ and $\psi$, i.e., $\mathcal{Z}_{\theta,\psi}(\cdot|s,a)=\mathcal{N}(Q_{\theta}(s,a),{\sigma_{\psi}(s,a)}^2)$.

Similar to standard Q-learning, we first define a random greedy target $y_D=r+\gamma Z(s',a'^*)$, where $a'^*=\arg\max_{a'}Q_{\theta}(s',a')$. Suppose $y_D\sim \mathcal{Z}^{\rm{target}}(\cdot|s,a)$, which is also assumed to be a Gaussian distribution. Note that even if $Z(s,a)$ and $y_D$ are not strictly Gaussian, we can still use the Gaussian to approximate their distributions, which will not affect the subsequent analysis. Since $\mathbb{E}[y_D]=\mathbb{E}[r]+ \gamma\mathbb{E}_{s'}[\max_{a'}Q_{\theta}(s',a')]$ is equal to $y$ in \eqref{eq.approximate_theta}, it follows  $\mathcal{Z}^{\rm{target}}(\cdot|s,a)=\mathcal{N}(y,{\sigma^{\rm{target}}}^2)$. Considering the loss function in \eqref{eq.distributional_Bellman} under the KL divergence measurement, $Q_{\theta}(s,a)$ and $\sigma_{\psi}(s,a)$ are updated by minimizing
\begin{equation}
\label{eq.update_rule_theory}
\begin{aligned}
&D_{\rm{KL}}(\mathcal{Z}^{\rm{target}}(\cdot|s,a),\mathcal{Z}_{\theta,\psi}(\cdot|s,a))\\
&\quad=\log\frac{\sigma_{\psi}(s,a)}{\sigma^{\rm{target}}}+\frac{{\sigma^{\rm{target}}}^2+(y-Q_{\theta}(s,a))^2}{2{\sigma_{\psi}(s,a)}^2}-\frac{1}{2},
\end{aligned}
\end{equation}
that is,
\begin{equation}
\label{eq.phi_appro_DRL}
\begin{aligned}
&\theta_{\rm{new}}=\theta+\beta\frac{y-Q_{\theta}(s,a)}{{\sigma_{\psi}(s,a)}^2}\nabla_{\theta}Q_{\theta}(s,a),\\
&\psi_{\rm{new}}=\psi+\beta\frac{\Delta\sigma^2+(y-Q_{\theta}(s,a))^2}{{\sigma_{\psi}(s,a)}^3}
 \nabla_{\psi}\sigma_{\psi}(s,a). 
\end{aligned}
\end{equation}
where $\Delta\sigma^2 = {\sigma^{\rm{target}}}^2-{\sigma_{\psi}(s,a)}^2$. Compared with standard Q-learning, $\sigma_{\psi}(s,a)$ plays a role of adaptively adjusting the update stepsize of $Q_{\theta}(s,a)$. In particular, the update stepsize of $Q_{\theta}(s,a)$ decreases squarely as $\sigma_{\phi}(s,a)$ increases. Supposing $Q_{\theta}(s,a)$ also obeys \eqref{eq.approximate_error}, the post-update parameters obtained based on the true target value $\tilde{y}$ is given by
\begin{equation}
\label{eq.phi_true_DRL}
\theta_{\rm{true}}=\theta+\beta \frac{\tilde{y}-Q_{\theta}(s,a)}{{\sigma_{\psi}(s,a)}^2} \nabla_{\theta} Q_{\theta}(s,a)
\end{equation}

Similar to the derivation of $\Delta(s,a)$, the overestimation bias of $Q_{{\theta}_{\rm{new}}}(s,a)$ in distributional Q-learning is
\begin{equation}
\label{eq.distri_over_bias}
\Delta_D(s,a) \approx \frac{\beta \gamma \delta \|\nabla_{\theta} Q_{\theta}(s,a)\|_2^2 }{{\sigma_{\psi}(s,a)}^2}
=\frac{\Delta(s,a)}{{\sigma_{\psi}(s,a)}^2}.
\end{equation}
Obviously, the overestimation errors $\Delta_D(s,a)$ is inversely proportional to $\sigma_{\psi}(s,a)^2$. In an ideal situation, when $\tilde{Q}(s,a)=\tilde{y}$, that is, $\tilde{Q}(s,a)$ has converged after a period of learning, we can derive that 
\begin{equation}
\nonumber
\begin{aligned}
&\mathbb{E}_{\epsilon_Q,\epsilon_Q'}[\sigma_{\psi_{\rm{new}}}(s,a)]\ge\sigma_{\psi}(s,a)+\\
&\quad\beta\frac{{\sigma^{\rm{target}}}^2-{\sigma_{\psi}(s,a)}^2+\gamma^2\delta^2+\mathbb{E}_{\epsilon_Q}[{\epsilon_Q}^2]}{{\sigma_{\psi}(s,a)}^3}\|\nabla_{\psi}\sigma_{\psi}(s,a)\|_2^2,
\end{aligned}
\end{equation}
where this inequality holds approximately since we drop higher order terms out in Taylor approximation.
See Appendix \ref{appen.derivation_std} for details of derivation. 

Because $\sigma_{\psi_{\rm{new}}}$ is also the standard deviation for the next time step, this indicates that by repeatedly applying \eqref{eq.phi_appro_DRL},  the standard deviation $\sigma_{\psi}(s,a)$ of the return distribution tends to be a larger value in areas with high $\sigma^{\rm{target}}$ and random errors $\epsilon_Q$. Moreover, $\sigma^{\rm{target}}$ is often positively related to the randomness of systems $p$, reward function $R$ and the return distribution $\mathcal{Z}(\cdot|s',a')$ of subsequent state-action pairs. Since the overestimation bias $\Delta_D(s,a)$ is inversely proportional to ${\sigma_{\psi}(s,a)}^2$ according to \eqref{eq.distri_over_bias}, distributional Q-learning can be used to mitigate overestimations caused by task randomness and approximation errors.

\section{Distributional Soft Actor-Critic}
\label{sec:DSAC}
\textcolor{black}{In this section, based on the developed DSPI framework, we derive the learning rules of the continuous return distribution, and propose the DSAC algorithm by replacing the clipped double Q-learning of SAC \cite{Haarnoja2018SAC,Haarnoja2018ASAC} with the return distribution learning.} We will consider a parameterized distributional value function $\mathcal{Z}_{\theta}(\cdot|s,a)$ and a stochastic policy $\pi_{\phi}(\cdot|s)$, where $\theta$ and $\phi$ are parameters. In this paper, both the state-action return distribution and policy functions are modeled as Gaussian with mean and covariance given by neural networks (NNs).  We will next derive update rules for parameters of these NNs. 
\subsection{Algorithm}
\label{sec.algorithm}
\subsubsection{Distributional Soft Policy Evaluation}
Considering the loss function in \eqref{eq.distributional_Bellman}, the soft state-action return distribution can be trained to minimize the loss function in \eqref{eq.distributional_Bellman} under the KL-divergence measurement
\begin{equation}
\label{eq.obective_general}
J_{\mathcal{Z}}(\theta)=  \mathop{\mathbb{E}}_{(s,a)\sim \mathcal{B}}\big[D_{\rm{KL}}(\mathcal{T}^{\pi_{\phi'}}_{\mathcal{D}}\mathcal{Z}_{\theta'}(\cdot|s,a),\mathcal{Z}_{\theta}(\cdot|s,a))\big]
\end{equation}
where $\mathcal{B}$ is a replay buffer of previously sampled experience, $\theta'$ and $\phi'$ are parameters of target return distribution and policy functions, which are used to stabilize the learning process and evaluate the target. \textcolor{black}{For practical applications, $\sigma^{\rm{target}}$ in \eqref{eq.update_rule_theory} is unknown. Therefore, we cannot directly update $\mathcal{Z}_{\theta}(\cdot|s,a)$ using the objective shown in \eqref{eq.update_rule_theory}. After analysis, we get the following objective function equivalent to \eqref{eq.obective_general}} 
\begin{equation}
\nonumber
J_{\mathcal{Z}}(\theta)= -\Exp_{\substack{(s,a,r,s')\sim\mathcal{B},a'\sim \pi_{\phi'},\\Z(s',a')\sim\mathcal{Z}_{\theta'}(\cdot|s',a')}}\Big[\log\mathcal{P}(\mathcal{T}^{\pi_{\phi'}}_{\mathcal{D}}Z(s,a)|\mathcal{Z}_{\theta}(\cdot|s,a))\Big].
\end{equation}
We provide details of derivation in Appendix \ref{appen.derivation_object}. 

The parameters $\theta$ can be optimized with the following gradients
\begin{equation}
\nonumber
\nabla_{\theta}J_{\mathcal{Z}}(\theta)= -\Exp_{\substack{(s,a,r,s')\sim\mathcal{B},\\a'\sim \pi_{\phi'},\\Z(s',a')\sim\mathcal{Z}_{\theta'}}}\Big[\nabla_{\theta}\log\mathcal{P}(\mathcal{T}^{\pi_{\phi'}}_{\mathcal{D}}Z(s,a)|\mathcal{Z}_{\theta}(\cdot|s,a))\Big].
\end{equation}
Since $\mathcal{Z}_{\theta}$ is assumed to be a Gaussian model, it can be expressed as $\mathcal{Z}_{\theta}(\cdot|s,a)=\mathcal{N}(Q_{\theta}(s,a),\sigma_{\theta}(s,a)^2)$, where $Q_{\theta}(s,a)$ and $\sigma_{\theta}(s,a)$ are the outputs of value network. This makes the Gaussian variant of update gradients
\begin{equation}
\nonumber
\begin{aligned}
&\nabla_{\theta}J_{\mathcal{Z}}(\theta)\\
&= -\Exp_{\substack{(s,a,r,s')\sim\mathcal{B},\\a'\sim \pi_{\phi'},\\Z(s',a')\sim\mathcal{Z}_{\theta'}}}\Big[\nabla_{\theta}\log\Big(\frac{e^{-\frac{\big(\mathcal{T}^{\pi_{\phi'}}_{\mathcal{D}}Z(s,a)-Q_{\theta}(s,a)\big)^2}{2{\sigma_{\theta}(s,a)}^2}}}{\sqrt{2\pi}\sigma_{\theta}(s,a)}\Big)\Big]\\
&=\Exp_{\substack{(s,a,r,s')\sim\mathcal{B},\\a'\sim \pi_{\phi'},\\Z(s',a')\sim\mathcal{Z}_{\theta'}}}\Big[\nabla_{\theta}\frac{\big(\mathcal{T}^{\pi_{\phi'}}_{\mathcal{D}}Z(s,a)-Q_{\theta}(s,a)\big)^2}{2{\sigma_{\theta}(s,a)}^2}+\frac{\nabla_{\theta}\sigma_{\theta}(s,a)}{\sigma_{\theta}(s,a)}\Big].
\end{aligned}
\end{equation}

Denoting $\Psi_{\mathcal{Z}}(\theta)=\log\mathcal{P}(\mathcal{T}^{\pi_{\phi'}}_{\mathcal{D}}Z(s,a)|\mathcal{Z}_{\theta}(\cdot|s,a))$, to understand the composition of $\nabla_{\theta}J_{\mathcal{Z}}(\theta)$ more intuitively, we can rewrite it as
\begin{equation}
\label{eq:distribution_gradient}
\begin{aligned}
\nabla_{\theta}J_{\mathcal{Z}}(\theta)=\Exp_{\substack{(s,a,r,s')\sim\mathcal{B},\\a'\sim \pi_{\phi'},\\Z(s',a')\sim\mathcal{Z}_{\theta'}}}\Big[-&\frac{\partial\Psi_{\mathcal{Z}}(\theta)}{\partial Q_{\theta}(s,a)}\nabla_{\theta}Q_{\theta}(s,a)\\
&\qquad-\frac{\partial\Psi_{\mathcal{Z}}(\theta)}{\partial{\sigma}_{\theta}(s,a)}\nabla_{\theta}{\sigma}_{\theta}(s,a)\Big],
\end{aligned}
\end{equation}
where 
\begin{equation}
\nonumber
\begin{aligned}
&\frac{\partial\Psi_{\mathcal{Z}}(\theta)}{\partial Q_{\theta}(s,a)}=\frac{\big(\mathcal{T}^{\pi_{\phi'}}_{\mathcal{D}}Z(s,a)-Q_{\theta}(s,a)\big)}{{\sigma_{\theta}(s,a)}^2},
\\
&\frac{\partial\Psi_{\mathcal{Z}}(\theta)}{\partial{\sigma}_{\theta}(s,a)}=\frac{\big(\mathcal{T}^{\pi_{\phi'}}_{\mathcal{D}}Z(s,a)-Q_{\theta}(s,a)\big)^2}{{\sigma_{\theta}(s,a)}^3}-\frac{1}{\sigma_{\theta}(s,a)}.
\end{aligned}
\end{equation}
It can be easily deduced from $\frac{\partial\Psi_{\mathcal{Z}}(\theta)}{\partial Q_{\theta}(s,a)}$ that the update stepsize of $Q_{\theta}(s,a)$ decreases squarely as $\sigma_{\theta}(s,a)$ increases, thereby mitigating Q-value overestimations. However, the gradients $\nabla_{\theta}J_{\mathcal{Z}}(\theta)$ are prone to explode as $\sigma_{\theta}(s,a) \rightarrow 0$, or to vanish as $\sigma_{\theta}(s,a) \rightarrow \infty$. To address this problem, we propose two options to keep $\sigma_{\theta}(s,a)$ within a reasonable range. The first point is to limit the minimum value of $\sigma_{\theta}(s,a)$ by
\begin{equation}
\label{eq.sigma_min}
\sigma_{\theta}(s,a)=\max(\sigma_{\theta}(s,a),\sigma_{{\rm{min}}}),
\end{equation}
Noted that if $\sigma_{{\rm{min}}}\ge1$, we always have $\Delta_D(s,a)\le \Delta(s,a)$. Therefore, in this paper, we let $\sigma_{{\rm{min}}}=1$. And the second point is to clip $\mathcal{T}^{\pi_{\phi'}}_{\mathcal{D}}Z(s,a)$ of $\frac{\partial\Psi_{\mathcal{Z}}(\theta)}{\partial{\sigma}_{\theta}(s,a)}$ to keep it close to the expectation value $Q_{\theta}(s,a)$ of the current soft return distribution, thus stabilizing the learning process of $\sigma_{\theta}(s,a)$ and indirectly controlling its range, i.e.,
\begin{equation}
\nonumber
\frac{\partial\Psi_{\mathcal{Z}}(\theta)}{\partial{\sigma}_{\theta}(s,a)}=\frac{\big(\overline{\mathcal{T}^{\pi_{\phi'}}_{\mathcal{D}}Z(s,a)}-Q_{\theta}(s,a)\big)^2}{{\sigma_{\theta}(s,a)}^3}-\frac{1}{\sigma_{\theta}(s,a)},
\end{equation}
where
\begin{equation}
\label{eq:clipping_bound}
\overline{\mathcal{T}^{\pi_{\phi'}}_{\mathcal{D}}Z(s,a)}={\rm{clip}}(\mathcal{T}^{\pi_{\phi'}}_{\mathcal{D}}Z(s,a),Q_{\theta}(s,a)-b,Q_{\theta}(s,a)+b),
\end{equation}
where $\text{clip}[x,A,B]$ denotes that $x$ is clipped into the range $[A,B]$ and $b$ is the clipping boundary.

The target networks mentioned above use a slow-moving update rate, parameterized by $\tau$, such as 
\begin{equation}
\nonumber
\theta' \leftarrow  \tau\theta+(1-\tau)\theta', \quad
\phi' \leftarrow  \tau\phi+(1-\tau)\phi'.
\end{equation}

\subsubsection{Distributional Soft Policy Improvement}
The policy can be learned by directly maximizing a parameterized variant of the objective in \eqref{eq.policy_imp}:
\begin{equation}
\nonumber
\begin{aligned}
&J_{\pi}(\phi)=\Exp_{s\sim\mathcal{B},a\sim\pi_{\phi}}[Q_{\theta}(s,a)-\alpha\log(\pi_{\phi}(a|s))]
\\
&=\Exp_{\substack{s\sim \mathcal{B},a\sim\pi_{\phi}}}\Big[\Exp_{Z(s,a)\sim\mathcal{Z}_{\theta}(\cdot|s,a)}[Z(s,a)]-\alpha\log(\pi_{\phi}(a|s))\Big].
\end{aligned}
\end{equation}
If $a$ is unbounded, given the parameters of the action distribution, such as the mean and variance of the Gaussian distribution, $\log(\pi_{\phi}(a|s))$ can be easily calculated. On the other hand, if $a$ is bounded to a finite interval, its log-likelihood can also be obtained in the manner given in Appendix \ref{appen.pdf_log_pi}.

There are several options, such as log derivative and reparameterization tricks, for maximizing $J_{\pi}(\phi)$ \cite{kingma2013repa}. In this paper, we apply the reparameterization trick because it can reduce the gradient estimation variance. 

If the soft Q-value function $Q_{\theta}(s,a)$ is explicitly parameterized through parameters $\theta$, we only need to express the random action $a$ as a deterministic variable, i.e.,
\begin{equation}
\label{eq.reparameter_a}
a=f_{\phi}(\xi_a;s),
\end{equation}
where $\xi_a \in \mathbb{R}^{{\rm{dim}}(\mathcal{A})}$ is an auxiliary variable which is sampled form some fixed distribution. In particular, since $\pi_{\phi}(\cdot|s)$ is assumed to be a Gaussian in this paper, $f_{\phi}(\xi_a;s)$ can be formulated as 
\begin{equation}
\nonumber
f_{\phi}(\xi_a;s)=a_{\rm{mean}}+\xi_a \odot a_{\rm{std}},
\end{equation}
where $a_{\rm{mean}}\in \mathbb{R}^{{\rm{dim}}(\mathcal{A})}$ and $a_{\rm{std}}\in \mathbb{R}^{{\rm{dim}}(\mathcal{A})}$ are the mean and standard deviation of $\pi_{\phi}(\cdot|s)$, $\odot$ represents the Hadamard product and ${\xi_a}_i \sim \mathcal{N}(0,\bf{I}_{{\rm{dim}}(\mathcal{A})})$.
Then the policy update gradients can be approximated with 
\begin{equation}
\nonumber
\begin{aligned}
\nabla_{\phi}J_{\pi}(\phi)&=\mathbb{E}_{s\sim \mathcal{B},\xi_a}\Big[-\alpha\nabla_{\phi}\log(\pi_{\phi}(a|s))+
\\&(\nabla_aQ_{\theta}(s,a)-\alpha\nabla_a\log(\pi_{\phi}(a|s)))\nabla_{\phi}f_{\phi}(\xi_a;s)\Big].
\end{aligned}
\end{equation}
If $Q_{\theta}(s,a)$ cannot be expressed explicitly through $\theta$, the policy update gradients can be obtained in the manner given in Appendix \ref{appen.policy_update}.

\subsubsection{Pseudo-code}
Finally, according to \cite{Haarnoja2018ASAC}, the temperature $\alpha$ is updated by minimizing the following objective
\begin{equation}
\nonumber
J(\alpha)=\mathbb{E}_{(s,a)\sim\mathcal{B}}[\alpha(- \log\pi_{\phi}(a|s)-\overline{\mathcal{H}})],
\end{equation}
where $\overline{\mathcal{H}}$ is the expected entropy. In addition, two-timescale updates, i.e., less frequent policy updates, usually result in higher quality policy updates \cite{Fujimoto2018TD3}. Therefore, 
the policy, temperature and target networks are updated every $m$ iterations in this paper. The final algorithm is listed in Algorithm \ref{alg:DSAC}. Fig. \ref{f:diagram} shows the diagram of DSAC. 
\begin{algorithm}[!htb]
\caption{DSAC Algorithm}
\label{alg:DSAC}
\begin{algorithmic}
\STATE Initialize parameters $\theta$, $\phi$ and $\alpha$
\STATE Initialize target parameters $\theta'\leftarrow\theta$, $\phi'\leftarrow\phi$
\STATE Initialize learning rate $\beta_{\mathcal{Z}}$, $\beta_{\pi}$, $\beta_{\alpha}$ and $\tau$ 
\STATE Initialize iteration index $k=0$
\REPEAT
\STATE Select action $a\sim\pi_{\phi}(a|s)$
\STATE Observe reward $r$ and new state $s'$
\STATE Store transition tuple $(s,a,r,s')$ in buffer $\mathcal{B}$
\STATE
\STATE Sample $N$ transitions $(s,a,r,s')$ from $\mathcal{B}$
\STATE Update soft return distribution $\theta \leftarrow \theta - \beta_{\mathcal{Z}}\nabla_{\theta}J_{\mathcal{Z}}(\theta)$
\IF{$k$ mod $m$}
\STATE Update policy $\phi \leftarrow \phi + \beta_{\pi}\nabla_{\phi} J_{\pi}(\phi)$
\STATE Adjust temperature $\alpha \leftarrow \alpha - \beta_{\alpha}\nabla_{\alpha} J(\alpha)$
\STATE Update target networks:
\STATE \qquad $\theta' \leftarrow  \tau\theta+(1-\tau)\theta'$, $\phi' \leftarrow  \tau\phi+(1-\tau)\phi'$
\ENDIF
\STATE $k=k+1$
\UNTIL Convergence  
\end{algorithmic}
\end{algorithm}

\begin{figure}[!htb]
\captionsetup{singlelinecheck = false,labelsep=period, font=small}
\centering{\includegraphics[width=0.45\textwidth]{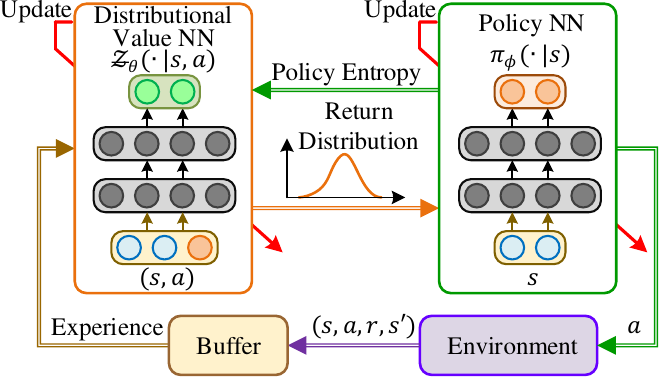}}
\caption{\textcolor{black}{DSAC diagram. The return distribution and policy are approximated by two NNs, called distributional value network and policy network respectively. DSAC first updates the distributional value network based on the samples collected from the buffer. Then, the output of the value network is used to guide the update of the policy network.}}
\label{f:diagram}
\end{figure}

\subsection{Architecture} \label{sec.architecture}
\textcolor{black}{Algorithm \ref{alg:DSAC} and Fig. \ref{f:diagram} show the operation process of DSAC in a serial way. Like most off-policy RL algorithms, we can use parallel or distributed learning techniques to improve the learning efficiency of DSAC. Therefore, we build a new parallel asynchronous buffer-actor-learner architecture (PABAL) referring to the other high-throughput learning architectures, such as IMPALA and Ape-X \cite{Espeholt2018IMPALA,horgan2018Ape-X,mnih2016A3C}.} As shown in Fig. \ref{f:architeture}, buffers, actors and learners are all distributed across multiple workers, which are used to improve the efficiency of storage and sampling, exploration, and updating, respectively. And all communication between modules is asynchronous.

Both actors and learners asynchronously synchronize the parameters from the shared memory. The experience generated by each actor is asynchronously and randomly sent to a certain buffer at each time step. Each buffer continuously stores data and sends the sampled experience to a random learner. Relying on the received sampled data, the learners calculate the update gradients using their local functions, and then use these gradients to update the shared value and policy functions. In this paper, we implement DSAC and other off-policy baseline algorithms within the PABAL architecture.  

\begin{figure}[!htb]
\captionsetup{singlelinecheck = false,labelsep=period, font=small}
\centering{\includegraphics[width=0.48\textwidth]{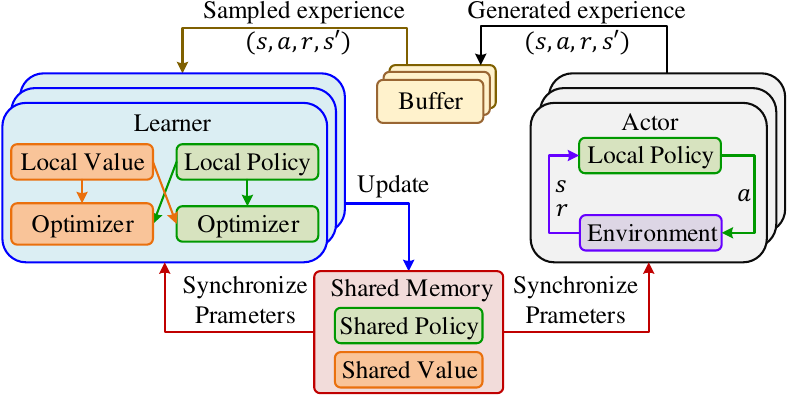}}
\caption{The PABAL architecture. Buffers, actors, and learners are all distributed across multiple workers. Communication between different modules is asynchronous.}
\label{f:architeture}
\end{figure}

\section{Experimental Verification}
\label{sec:experiments}

\subsection{Benchmarks}
To evaluate our algorithm, we measure its performance and Q-value estimation bias on a suite of MuJoCo continuous control tasks without modifications to environment \cite{todorov2012mujoco}, interfaced through OpenAI Gym \cite{brockman2016openaigym}. Fig. \ref{fig.env} shows the benchmark tasks used in this paper. \textcolor{black}{See Appendix \ref{appen:benchmarks}
for brief descriptions of these benchmarks.}
\begin{figure}[!htb]
\centering
\captionsetup[subfigure]{justification=centering}
\subfloat[]{\label{subfig:Human_env}\includegraphics[width=0.15\textwidth]{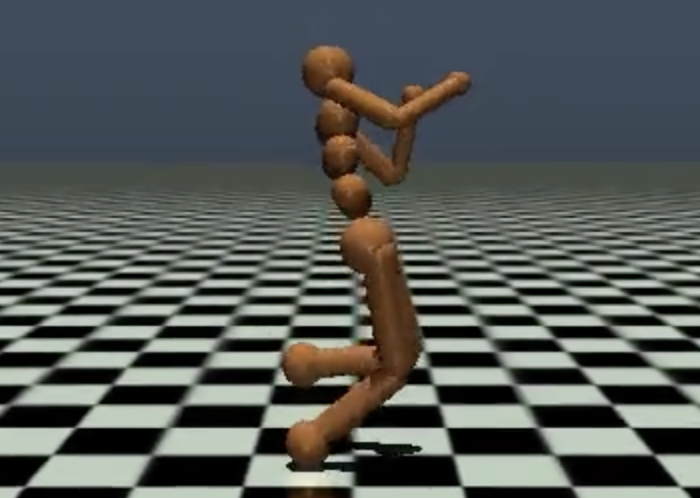}} 
\subfloat[]{\label{subfig:Halfcheetah_env}\includegraphics[width=0.15\textwidth]{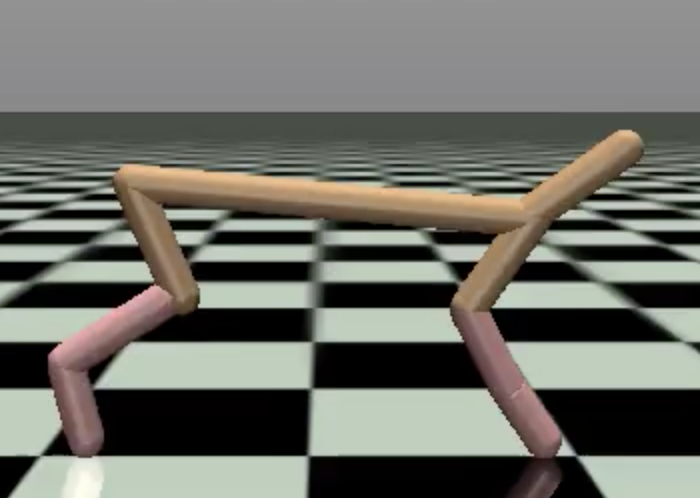}} 
\subfloat[]{\label{subfig:Ant_env}\includegraphics[width=0.15\textwidth]{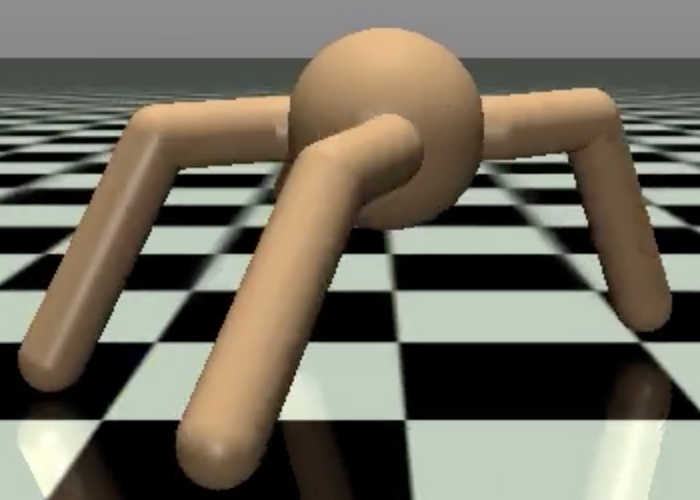}} \\
\subfloat[]{\label{subfig:Walker_env}\includegraphics[width=0.15\textwidth]{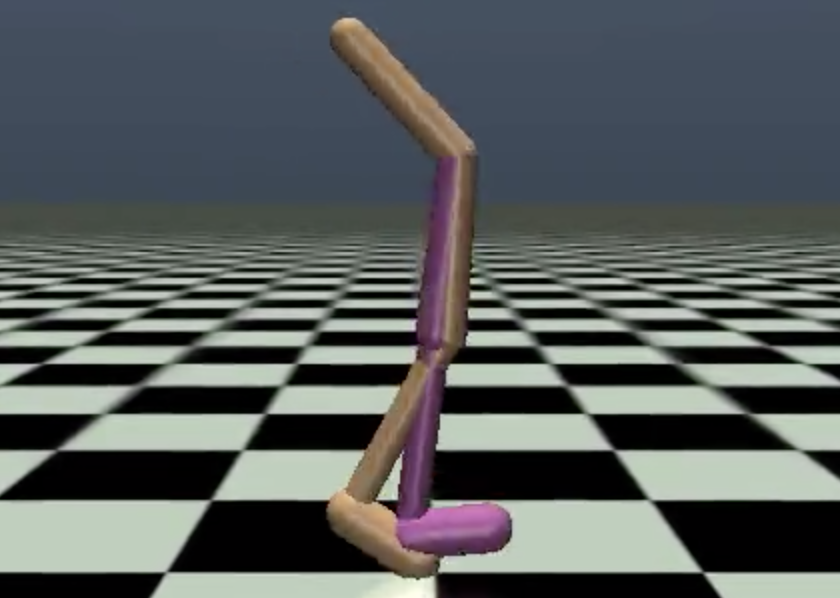}} 
\subfloat[]{\label{subfig:Invert_env}\includegraphics[width=0.15\textwidth]{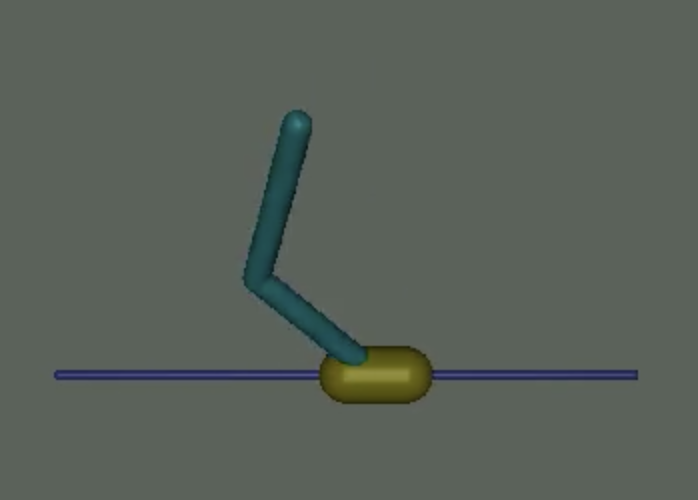}} 
\caption{Tasks. (a) Humanoid-v2: $(s\times a)\in \mathbb{R}^{376} \times  \mathbb{R}^{17}$. (b) HalfCheetah-v2: $(s\times a)\in \mathbb{R}^{17} \times \mathbb{R}^{6}$. (c) Ant-v2: $(s\times a)\in  \mathbb{R}^{111}\times \mathbb{R}^{8}$. (d) Walker2d-v2: $(s\times a)\in \mathbb{R}^{17} \times \mathbb{R}^{6}$. (e) InvertedDoublePendulum-v2: $(s\times a)\in \mathbb{R}^{11}\times\mathbb{R}^{1}$.}
\label{fig.env}
\end{figure}

\subsection{Baselines}
\textcolor{black}{We compare our algorithm against Deep Deterministic Policy Gradient (DDPG) \cite{lillicrap2015DDPG}, Trust Region Policy Optimization (TRPO) \cite{schulman2015TRPO}, Proximal Policy Optimization (PPO) \cite{schulman2017PPO}, Distributed Distributional Deep Deterministic Policy Gradients (D4PG) \cite{barth-maron2018D4PG}, Twin Delayed Deep Deterministic policy gradient (TD3) \cite{Fujimoto2018TD3}, Soft Actor-Critic (SAC) \cite{Haarnoja2018ASAC}.  DDPG, TRPO, PPO, D4PG, TD3 and SAC are mainstream RL algorithms, which have been extensively verified and applied in a variety of challenging domains. Using these algorithms as baselines, the performance of the proposed DSAC algorithm can be evaluated objectively.} 

\textcolor{black}{We additionally compare our method with our proposed Twin Delayed Distributional Deep Deterministic policy gradient algorithm (TD4), which is developed by replacing the clipped double Q-learning in TD3 with the distributional return learning; Double Q-learning variant of SAC (Double-Q SAC), in which we replace the clipped double Q-learning of SAC with the actor-critic variant of double Q-learning \cite{hasselt2010double_Q,Fujimoto2018TD3}; and single Q-value variant of SAC (Single-Q SAC), in which we replace the clipped double Q-learning of SAC with traditional TD learning.  See Appendix \ref{appen:Double-Q SAC Algorithm}, \ref{appen:single-Q SAC} and \ref{appen:TD4} for detailed descriptions of Double-Q SAC, Single-Q SAC and TD4 algorithms. Double-Q SAC and Single-Q SAC are adapted from SAC. Table \ref{table.baselines} gives a basic description of DSAC and each baseline. It is clear that DSAC, SAC, Double-Q SAC and Single-Q SAC algorithms respectively use the return distribution learning, clipped double Q-learning, double Q-learning and traditional TD learning for policy evaluation.  This is the only difference between these algorithms. Therefore, we can assess the impact of the distribution learning by comparing DSAC with SAC, Single-Q SAC and Double-Q SAC. Besides, we compare DSAC with TD4, which uses the distribution learning but not maximum entropy, to assess the impact of policy entropy.}

\begin{table*}[!htb]
\centering
\captionsetup{justification=centering,labelsep=newline,font={small,sc}}
\caption{\textcolor{black}{Basic description of DSAC and baselines.}}
\label{table.baselines}
{\color{black}{
\begin{tabular}{lllll}
\toprule
Algorithm& Algorithm Type & Policy Type & Policy Evaluation & Policy Improvement \\
\hline
DSAC (Ours) & off-policy & Stochastic & Continuous soft return distribution learning & Soft policy gradient \\
SAC\cite{Haarnoja2018ASAC} & off-policy &  Stochastic & Clipped double Q-learning & Soft policy gradient \\
Double-Q SAC &off-policy &  Stochastic &Double Q-learning & Soft policy gradient\\
Single-Q SAC &off-policy&  Stochastic & Traditional TD learning &  Soft policy gradient\\
TD4  &off-policy& Deterministic& Continuous return distribution learning &  Policy gradient \\
TD3\cite{Fujimoto2018TD3}  &off-policy& Deterministic& Clipped double Q-learning &  Policy gradient \\
DDPG\cite{lillicrap2015DDPG}  &off-policy& Deterministic& Traditional TD learning &  Policy gradient \\
D4PG\cite{barth-maron2018D4PG}  &off-policy& Deterministic& Discrete return distribution learning &  Policy gradient \\
TRPO\cite{schulman2015TRPO}  & on-policy & Stochastic& Traditional TD learning &  Constrained Policy Optimization \\
PPO\cite{schulman2017PPO}  &on-policy & Stochastic& Traditional TD learning &  Proximal Policy Optimization \\
\bottomrule
\end{tabular}}}
\end{table*}

All the off-policy algorithms mentioned above are implemented in the proposed PABAL architecture, including 4 learners, 6 actors and 3 buffers. We use a fully connected network with 5 hidden layers, consisting of 256 units per layer, with Gaussian Error Linear Units (GELU) each layer \cite{hendrycks2016gelu}, for both actor and critic. For distributional value function and stochastic policy, we use a Gaussian distribution with mean and covariance given by a NN, where the covariance matrix is diagonal.  In this case, each NN maps the input states to the mean and logarithm of standard deviation of the Gaussian distribution. The Adam method \cite{Diederik2015Adam} with a cosine annealing learning rate is used to update all the parameters. All algorithms adopt almost the same NN architecture and hyperparameters. Table \ref{table.hyper} in Appendix \ref{appen.hyper} provides more detailed hyperparameters of all algorithms. 

\subsection{Results}
\subsubsection{Performance}
\begin{figure*}[!htb]
\centering
\captionsetup[subfigure]{justification=centering}
\subfloat[Humanoid-v2]{\label{fig:Humanoid_ave}\includegraphics[width=0.33\textwidth]{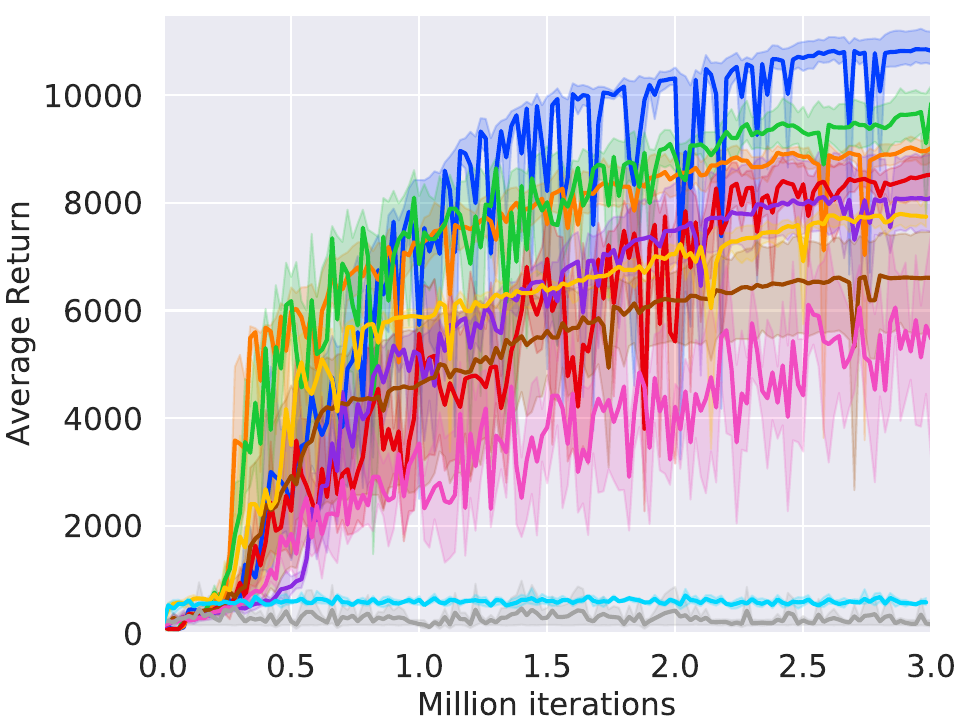}} 
\subfloat[Ant-v2]{\label{fig:Ant_ave}\includegraphics[width=0.33\textwidth]{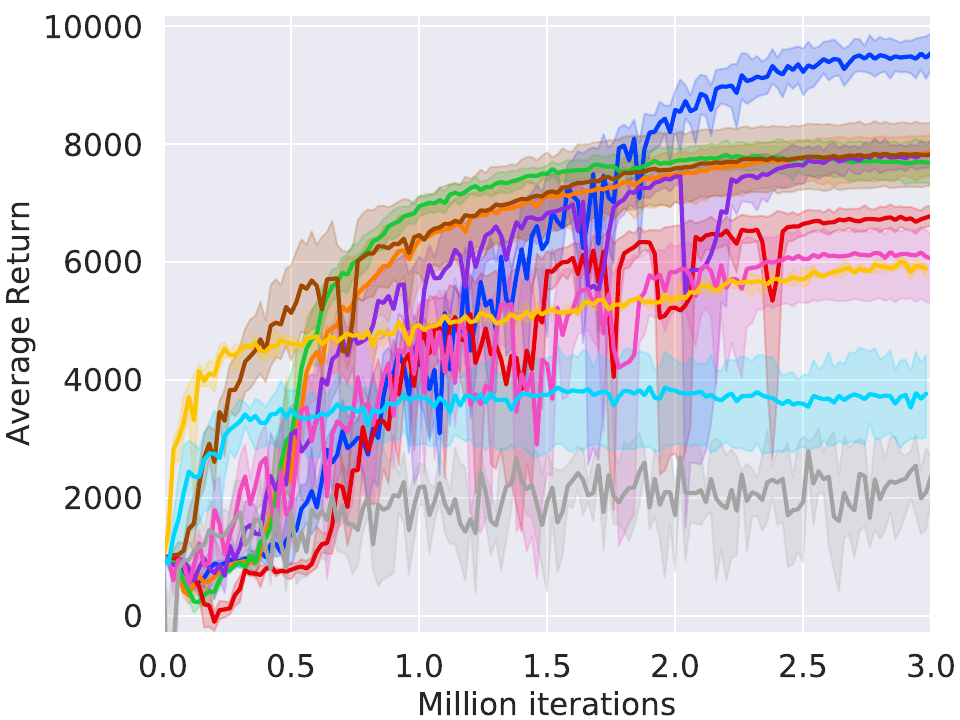}}
\subfloat[Walker2d-v2]{\label{fig:Walker2d_ave}\includegraphics[width=0.33\textwidth]{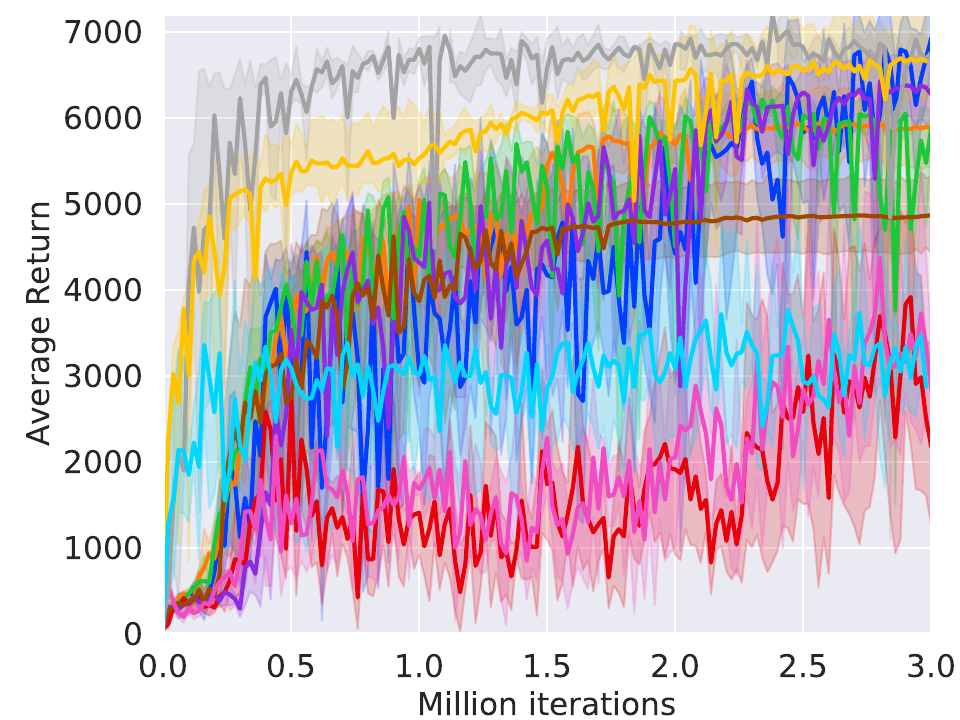}} \\
\subfloat[HalfCheetah-v2]{\label{fig:HalfCheetah_ave}\includegraphics[width=0.33\textwidth]{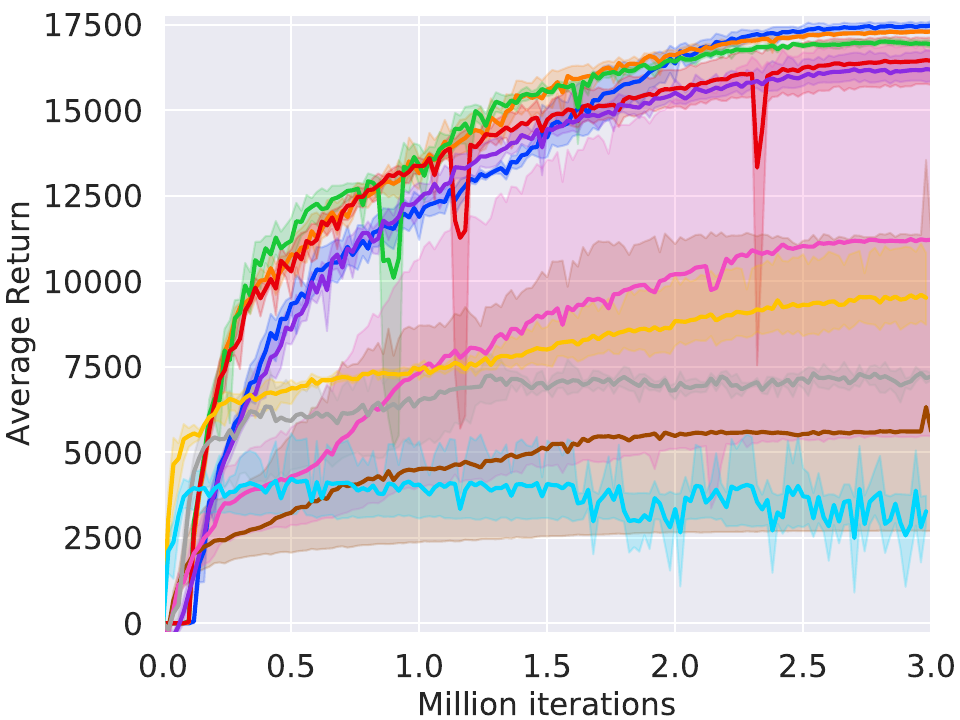}} 
\subfloat[InvertedDoublePendulum-v2]{\label{fig:InvertedDoublePendulum_ave}\includegraphics[width=0.33\textwidth]{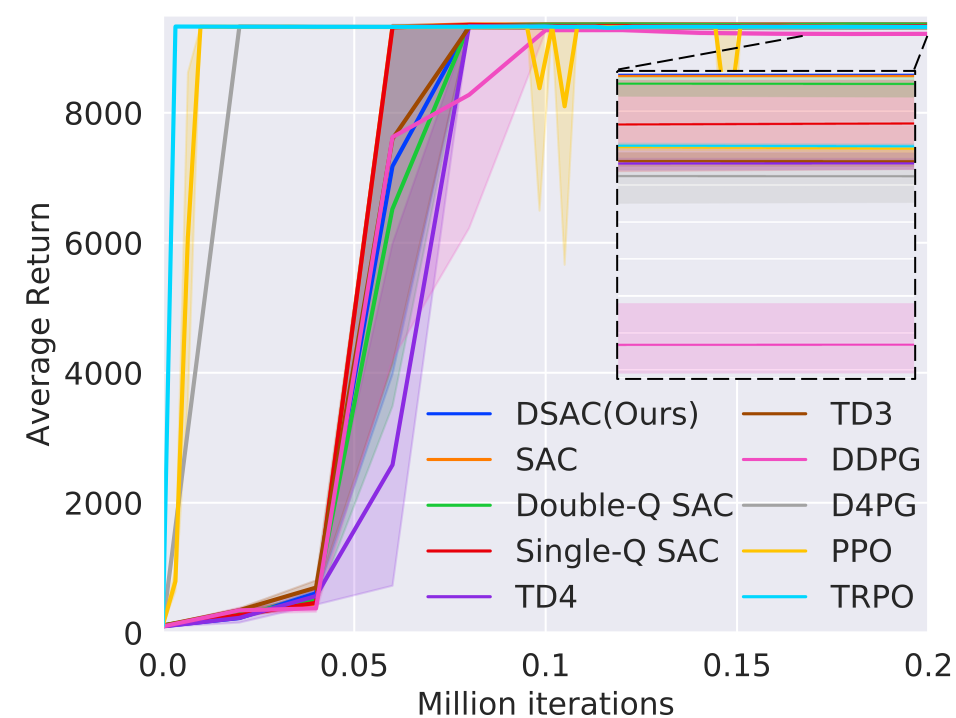}} 
\caption{\textcolor{black}{Training curves on continuous control benchmarks. The solid lines correspond to the mean and the shaded regions correspond to 95\% confidence interval over 5 runs.}}
\label{fig.average_return}
\end{figure*}
We train 5 different runs of each algorithm with different random seeds, with evaluations every 20000 iterations. Each evaluation calculates the average return over 5 episodes without exploration noise, where the maximum length of each episode is 1000 time steps.  
The learning curves are shown in Fig. \ref{fig.average_return} and results in Table \ref{table.ave_return}. \textcolor{black}{Results show that the proposed DSAC algorithm outperforms or matches all other baseline algorithms across all benchmark tasks in terms of the final performance. For example, compared with famous RL algorithms such as SAC, TD3, PPO, and DDPG, DSAC gains 20.0\%, 63.8\%, 39.8\%, 97.6\% improvements on the most complex Humanoid-v2 task, respectively. This indicates that the final performance of DSAC on these benchmarks exceeds the state of the art. Fig. \ref{fig.dsac_vs_sac} visually shows the control performance of DSAC and SAC on Humanoid-v2. It is obvious that DSAC realizes a movement closer to human running. Among DSAC, SAC, Single-Q SAC and Double-Q SAC, DSAC has achieved the best performance on all tasks, which shows that the return distribution learning is an important measure to improve policy performance. Besides, TD4 also outperforms TD3 and DDPG on most tasks, which shows that algorithms with deterministic policies also benefit greatly from the return distribution learning. As TD4 exceeds the performance of D4PG, which learns a discrete return distribution, with a wide margin on Humanoid-v2, Ant-v2 and HalfCheetah-v2, this indicates that learning a continuous distribution causes significant performance improvements in most cases. Compared with TD4, DSAC achieves 33.8\%, 22.1\%, 10.4\%, 8.0\% improvements on Humanoid-v2, Ant-v2, Walker2d-v2, and HalfCheetah-v2, respectively, suggesting that the maximum entropy framework is an effective measure to achieve good performance. 
}

\begin{table*}[!htb]
\centering
\captionsetup{justification=centering,labelsep=newline,font=small}
\captionsetup{justification=centering,labelsep=newline,font={small,sc}}
\caption{\textcolor{black}{Average final return. Maximum value for each task is bolded. $\pm$ corresponds to a single standard deviation over 5 runs.}}
\label{table.ave_return}
\textcolor{black}{
\begin{tabular}{cccccc}
\toprule
Task& Humanoid-v2   &  Ant-v2& Walker2d-v2  &HalfCheetah-v2 &  InvDoublePendulum-v2\\
\hline
DSAC (Ours) & \textbf{10824}$\pm$\textbf{347} & \textbf{9547}$\pm$\textbf{346}  & \textbf{6920}$\pm$\textbf{405} & \textbf{17479}$\pm$\textbf{148} & \textbf{9359.7}$\pm$\textbf{0.2} \\
SAC          & 9019$\pm$292         & 7856$\pm$416                      & 5878$\pm$580        & 17300$\pm$39               & 9359.6 $\pm$0.2\\
Double-Q SAC & 9844$\pm$396          & 7682$\pm$428                    & 5881$\pm$227        & 16926$\pm$132             & 9359.4$\pm$0.6\\
Single-Q SAC & 8525$\pm$488         & 6783$\pm$197                      & 2176$\pm$1251       & 16445$\pm$815              & 9355.2$\pm$3.6\\
TD4          & 8090$\pm$789         & 7821$\pm$262                  & 6270$\pm$435      & 16187$\pm$538          & 9320.2$\pm$18.3\\
TD3          & 6610$\pm$1062       &7828$\pm$642                     & 4864$\pm$512        & 5619$\pm$5779          & 9315.5$\pm$10.4\\
DDPG         & 5477$\pm$2438       & 6060$\pm$747                  & 2849$\pm$690          & 11214$\pm$6861       & 9198.0$\pm$13.1\\
D4PG          & 175$\pm$53         & 2367$\pm$303                  & 6588$\pm$260      & 7215$\pm$89          & 9300.9$\pm$16.3\\
PPO         & 7743$\pm$267       &5889$\pm$111                     & 6654$\pm$492        & 9517$\pm$936          & 9318.7$\pm$0.7\\
TRPO         & 581$\pm$56       & 3767$\pm$573                 & 2870$\pm$28          & 3274$\pm$346       & 9324.6$\pm$2.8\\
\bottomrule
\end{tabular}}
\end{table*}

\begin{figure*}[!htb]
\centering
\captionsetup[subfigure]{justification=centering}
\subfloat[DSAC (Ours)]{\label{fig:dsac_human}\includegraphics[width=0.65\textwidth]{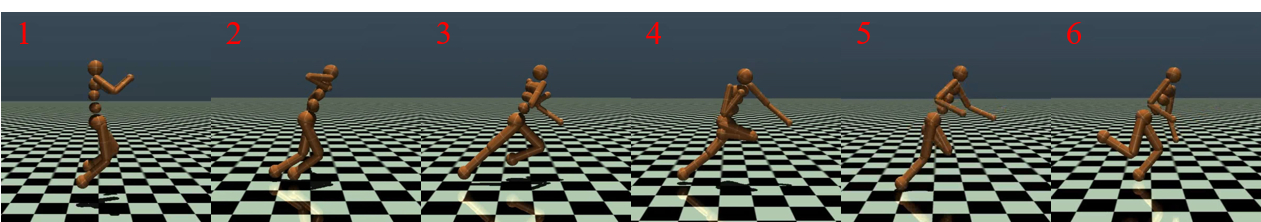}} \\
\subfloat[SAC]{\label{fig:sac_human}\includegraphics[width=0.65\textwidth]{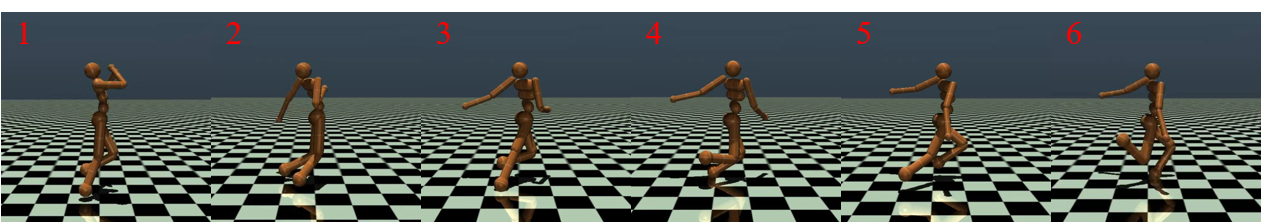}}
\caption{\textcolor{black}{DSAC vs SAC on Humanoid-v2.}}
\label{fig.dsac_vs_sac}
\end{figure*}

\subsubsection{Q-value Estimation Accuracy}

\textcolor{black}{To evaluate the impact of the return distribution learning on Q-value estimation accuracy, this section compares the estimation bias of DSAC, SAC, Double-Q SAC and Single-Q SAC on different benchmarks. The Q-value estimation bias is equal to the difference between the Q-value estimate and the true Q-value. To approximate the true Q-value, we calculate the average actual discounted return over states of 10 episodes every 20000 iterations (evaluate up to the first 200 states per episode). Fig. \ref{fig.average_bia} graphs the average Q-value estimation and true Q-value curves during learning. Table \ref{table.bias} gives the average relative Q-value estimation bias which equals the Q-value estimation bias divided by the true Q-value. Noted that this part excludes the InvDoublePendulum-v2 task, because due to its simplicity, a good policy has been learned before the value function converges.}

\begin{figure*}[!htb]
\centering
\captionsetup[subfigure]{justification=centering}
\subfloat[Humanoid-v2]{\label{fig:Humanoid_bias}\includegraphics[width=0.25\textwidth]{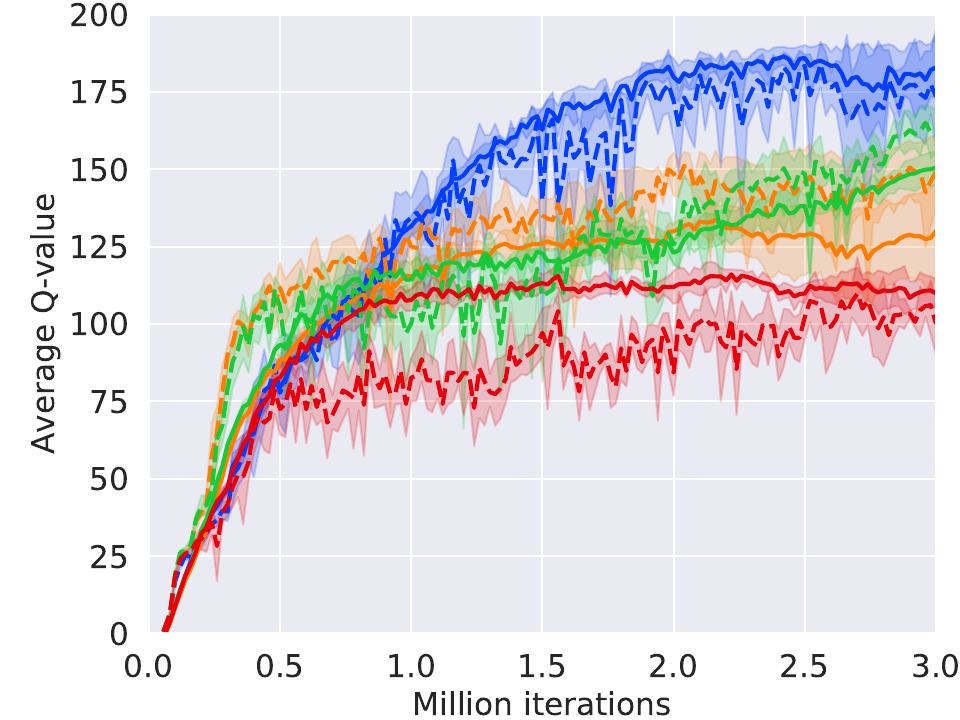}} 
\subfloat[Ant-v2]{\label{fig:Ant_bias}\includegraphics[width=0.25\textwidth]{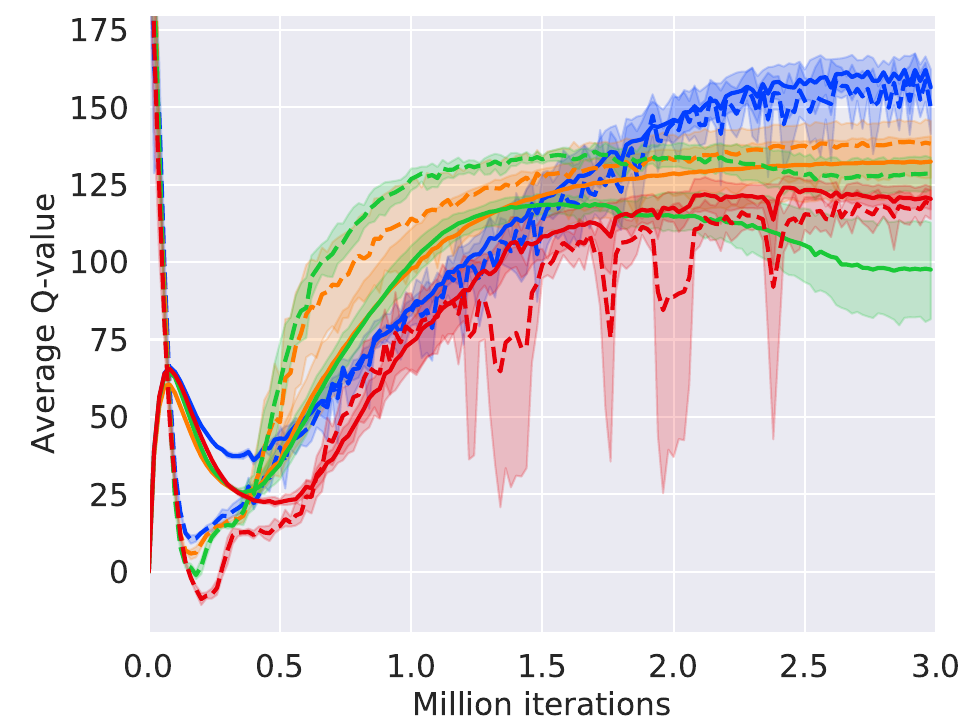}} 
\subfloat[Walker2d-v2]{\label{fig:Walker2d_bias}\includegraphics[width=0.25\textwidth]{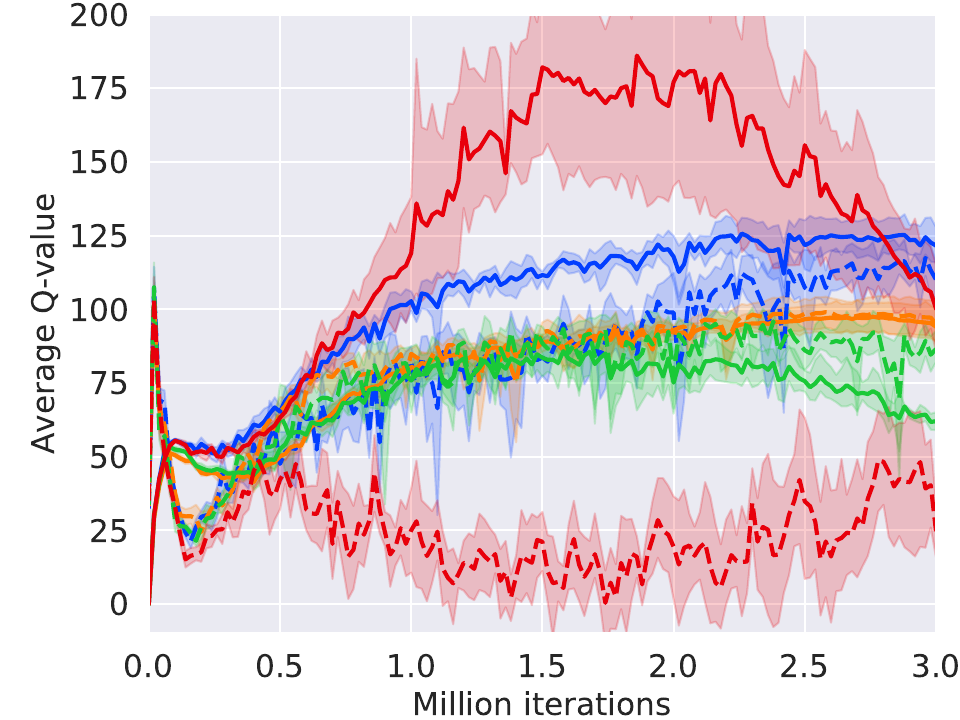}} 
\subfloat[HalfCheetah-v2]{\label{fig:HalfCheetah_bias}\includegraphics[width=0.25\textwidth]{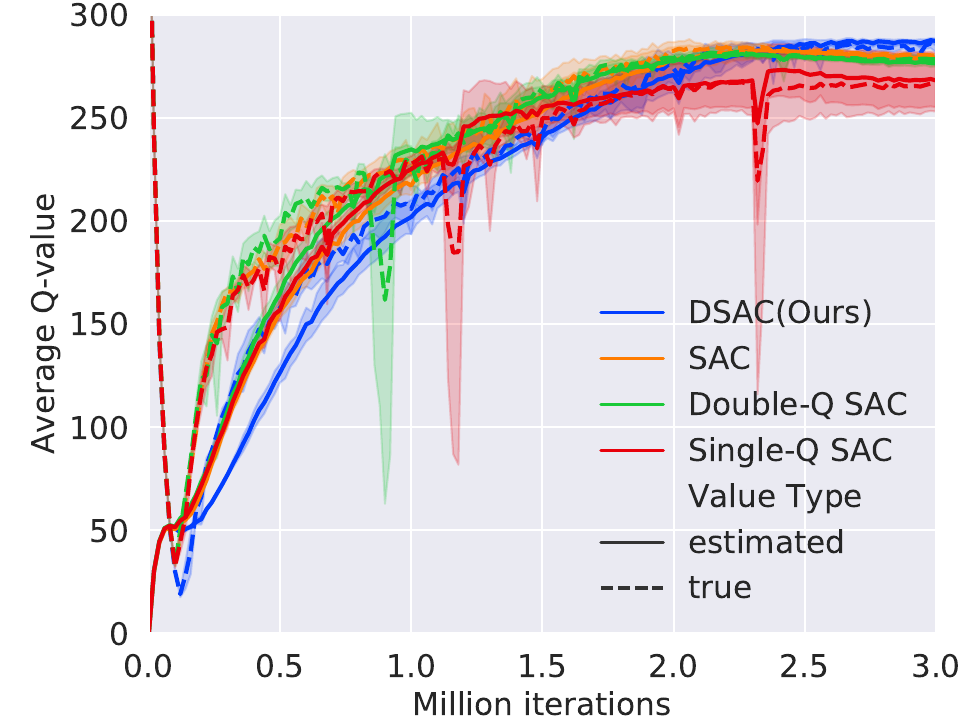}} 
\caption{\textcolor{black}{Average true Q-value vs estimated Q-value. The solid lines correspond to the mean and the shaded regions correspond to 95\% confidence interval over 5 runs.}}
\label{fig.average_bia}
\end{figure*}

\begin{table*}[!htb]
\centering
\captionsetup{justification=centering,labelsep=newline,font=small}
\captionsetup{justification=centering,labelsep=newline,font={small,sc}}
\caption{\textcolor{black}{Average relative Q-value estimation bias over 5 runs. We average the relative estimation bias from 1.5 million to 3 million iterations for each run. $+$ and $-$ indicate overestimation and underestimation, respectively. $\pm$ corresponds to a single standard deviation over 5 runs.}}
\label{table.bias}
{\color{black}{
\begin{tabular}{cccccc}
\toprule
Algorithm &Main difference & Humanoid-v2 &  Ant-v2& Walker2d-v2 &HalfCheetah-v2   \\
\hline
DSAC (Ours) &Continuous return distribution learning& +5.32\%$\pm$0.62\%& +3.48\%$\pm$0.69\% &  +17.71\%$\pm$2.30\% & -0.33\%$\pm$0.18\% \\
Single-Q SAC &Traditional TD learning&  +15.85\%$\pm$1.06\%& +9.24\%$\pm$5.74\%  &+943.80\%$\pm$683.94\% &+1.56\%$\pm$1.67\% \\
SAC &Clipped double Q-learning& -10.16\%$\pm$1.37\% & -4.07\%$\pm$0.66\% & -1.45\%$\pm$1.06\%  & -0.99\%$\pm$0.66\%\\
Double-Q SAC & Double Q-learning &  -4.63\%$\pm$1.70\%  & -16.68\%$\pm$4.21\%    & -12.84\%$\pm$4.03\%         &-0.33\%$\pm$0.32\%\\
\bottomrule
\end{tabular}}}
\end{table*}

\textcolor{black}{Compared with Single-Q SAC that updates Q-value using the traditional TD learning method, the overestimation bias of DSAC is reduced by 10.53\%, 5.76\%, 926.09\%, 1.89\% on Humanoid-v2, Ant-v2, Walker2d-v2, and HalfCheetah-v2, respectively. Our results demonstrate the theoretical analysis in Section \ref{sec:overesimation_in_distributional}, i.e., the return distribution learning can be used to reduce overestimations without introducing 
any additional value or policy network.  As a comparison, SAC (uses clipped double Q-learning) and Double-Q SAC (uses double Q-learning) suffer from underestimations during the learning procedure. While the effect of each value learning method varies from task to task, the Q-value estimation accuracy of DSAC is higher than SAC and Double-Q SAC in most cases. This explains why DSAC exceeds Single-Q SAC, SAC, and Double-Q SAC on most benchmarks by a wide margin. Therefore, our results demonstrate that the return distribution learning can greatly improve policy performance by mitigating overestimations.}

\subsubsection{Time Efficiency}
Fig. \ref{f:time} compares the time efficiency of different off-policy algorithms. Results show that the average wall-clock time consumption per 1000 iterations of DSAC is comparable to DDPG, and much lower than SAC, TD3, and Double-Q SAC. This is because that unlike double Q-learning and clipped double Q-learning, the return distribution learning does not need to introduce any additional value network or policy network (excluding target networks) to reduce overestimations. 
\begin{figure}[!htb]
\captionsetup{justification =raggedright,
              singlelinecheck = false,labelsep=period, font=small}
\centering{\includegraphics[width=0.4\textwidth]{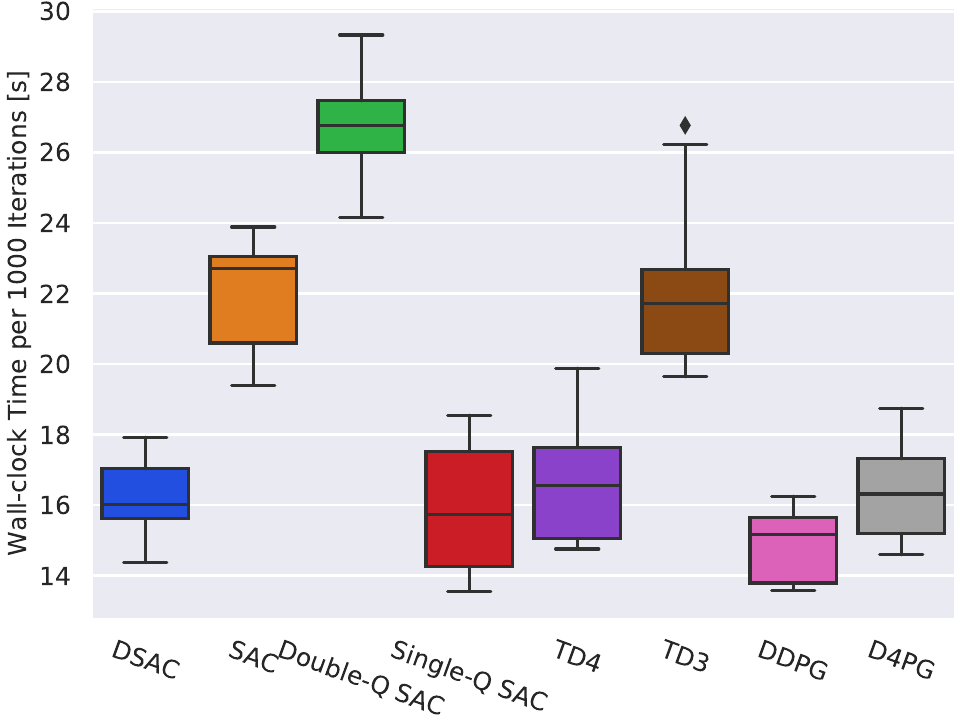}}
\caption{Algorithm comparison in terms of time efficiency on the Ant-v2 benchmark. Each boxplot is drawn based on values of 50 evaluations. All evaluations were performed on a single computer with a 2.4 GHz 20 core Intel Xeon CPU.}
\label{f:time}
\end{figure}

{\color{black}{
\subsection{Ablation Studies}

As shown in Table \ref{table.hyper}, compared with SAC, DSAC introduces two hyperparameters: 1) the minimum standard deviation $\sigma_{\text{min}}$ in \eqref{eq.sigma_min}, and 2) the clipping boundary $b$ in \eqref{eq:clipping_bound}. These two hyperparameters are employed to prevent exploding and vanishing gradient problems when learning the continuous distributional value function $\mathcal{Z}_{\theta}(\cdot|s,a)$.

We first take the Ant-v2 task as an example to analyze the influence of $\sigma_{{\rm{min}}}$ on the final performance. From \eqref{eq:distribution_gradient},  the gradients $\nabla_{\theta}J_{\mathcal{Z}}(\theta)$ are prone to explode as $\sigma_{\theta}(s,a) \rightarrow 0$. Therefore, $\sigma_{\theta}(s,a)$ should be bounded above by a specific positive value. Besides, according to the analysis in Section \ref{sec:overesimation_in_distributional}, if $\sigma_{{\rm{min}}}\ge1$, we always have $\Delta_D(s,a)\le \Delta(s,a)$. But a too large $\sigma_{{\rm{min}}}$ may reduce the estimation accuracy of the return distribution. Therefore, this paper sets $\sigma_{{\rm{min}}}=1$. Fig. \ref{fig:ablation_sigma} graphs the average final return of DSAC under different $\sigma_{\rm{min}}$ values on Ant-v2. Our results show that when $\sigma_{\rm{min}}=1$, DSAC achieves the best final performance on Ant-v2, which is consistent with the above analysis.

We additionally perform the ablation study to compare the performance of DSAC with different clipping boundaries $b$. Our results are presented in Fig. \ref{fig:ablation_bound}. In this paper, the clipping boundary $b$ is employed to stabilize the learning process of $\sigma_{\theta}(s,a)$ and keep it in a reasonable range.  Results indicate that compared with the performance of removing the clipping boundary trick from DSAC (i.e., $b=+\infty$), the inclusion of $b$ (for different $b$ values) generally improves performance. Therefore, DSAC appears to benefit greatly from the clipping boundary trick. However, the final performance is a little bit sensitive to the value of $b$. This is because that a too small $b$ will reduce the learning accuracy of the return distribution, while a too large $b$ cannot effectively limit the range of $\sigma_{\theta}(s,a)$. In practical applications, it is usually necessary to select an appropriate $b$ value according to the range of the state-action return $Z(s,a)$, which limits the flexibility of the DSAC algorithm. We will focus on this issue in the future.}}

\begin{figure}[!htb]
\centering
\captionsetup[subfigure]{justification=centering}
\subfloat[Performance under different $\sigma_{\rm{min}}$]{\label{fig:ablation_sigma}\includegraphics[width=0.25\textwidth]{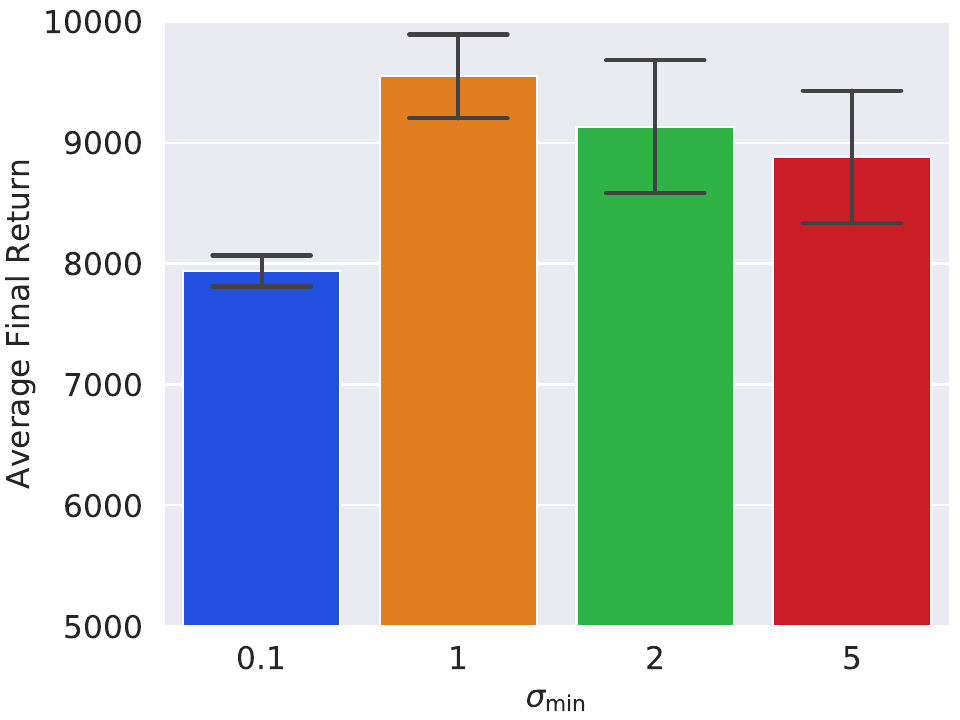}} 
\subfloat[Performance under different $b$]{\label{fig:ablation_bound}\includegraphics[width=0.25\textwidth]{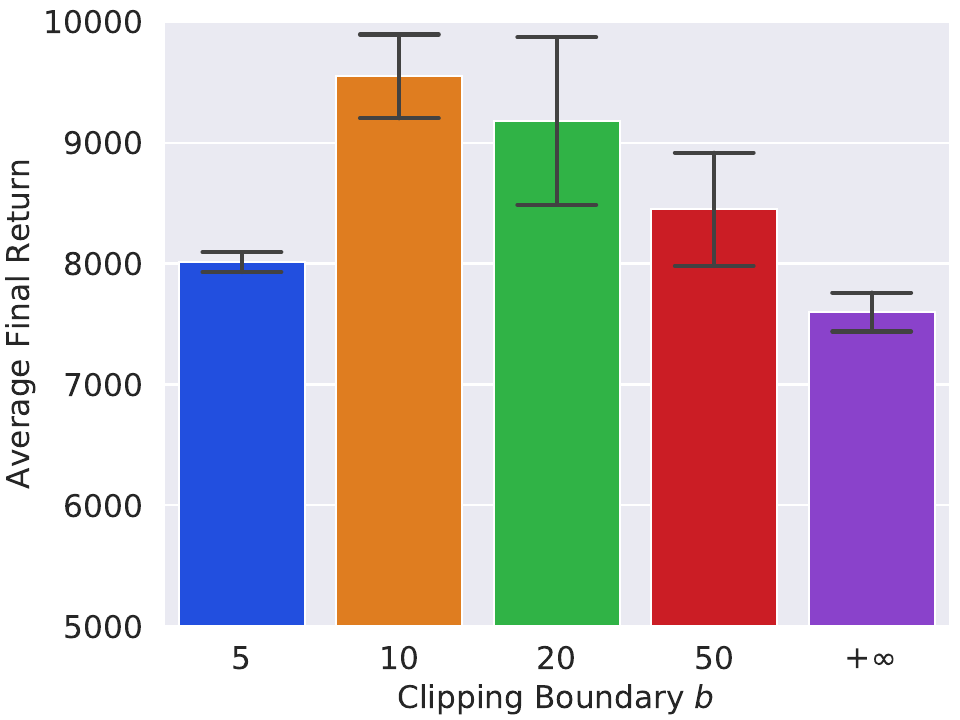}} 
\caption{\textcolor{black}{Average final return of DSAC under different hyperparameters on Ant-v2 over 5 runs. (a) $b=10$. (b) $\sigma_{\rm{min}}=1$.}}
\label{f.ablation}
\end{figure}

\section{Conclusions}
\label{sec:conclusion}
In this paper, we propose an off-policy RL algorithm for continuous control setting, called distributional soft actor-critic (DSAC), to mitigate Q-value overestimations, thereby improving policy performance. We first discover in theory that the update stepsize of the Q-value function in distributional RL decreases squarely as the standard deviation of state-action returns increases, thus mitigating Q-value overestimations. Then, a distributional soft policy iteration (DSPI) framework is developed by embedding the return distribution function into maximum entropy RL, which alternates  between distributional soft policy evaluation and soft policy improvement. Next, a deep off-policy  actor-critic variant of DSPI, 
i.e., DSAC, is proposed to directly learn a continuous return distribution by keeping the variance of the state-action returns within reasonable range to address exploding and vanishing gradient problems. We evaluate DSAC and 9 baselines (such as SAC, TD3, PPO, DDPG) on the suite of MuJoCo tasks. Results show that DSAC outperforms or matches all other baseline algorithms across all benchmarks. 
% In the unusual situation where you want a paper to appear in the
% references without citing it in the main text, use \nocite

\appendices
\section{Proof of Convergence of Distributional Soft Policy Iteration}
\label{appen.proof}
In this appendix, we present proofs to show that Distributional Soft Policy Iteration (DSPI), which alternates between \eqref{eq.soft_distri_bellman} and \eqref{eq.policy_imp}, would lead to policy improvement with respect to the maximum entropy objective. The proofs borrow heavily from the policy evaluation and policy improvement theorems of Q-learning, distributional RL and soft Q-learning \cite{sutton2018reinforcement,bellemare2017C51,Haarnoja2018SAC}.
\begin{lemma}\label{lemma.dspe}
(Distributional Soft Policy Evaluation). Consider the distributional soft bellman backup operator $\mathcal{T}^{\pi}_{\mathcal{D}}$ in \eqref{eq.soft_distri_bellman} and a soft state-action distribution function $\mathcal{Z}^{0}(Z^0(s,a)|s,a): \mathcal{S}\times\mathcal{A}\rightarrow \mathcal{P}(Z^0(s,a))$, which maps a state-action pair $(s,a)$ to a distribution over random soft state-action returns $Z^{0}(s,a)$, and define $Z^{i+1}(s,a)=\mathcal{T}^{\pi}_{\mathcal{D}}Z^{i}(s,a)$, where $Z^{i+1}(s,a) \sim \mathcal{Z}^{i+1}(\cdot|s,a)$. Then the sequence $\mathcal{Z}^{i}$ will converge to $\mathcal{Z}^{\pi}$ as $i\rightarrow \infty$.
%a mapping $Z^0: \mathcal{S}\times\mathcal{A}\rightarrow RV$, 
\end{lemma}
\begin{proof}
Let $\overline{Z}$ denote the space of soft return function $Z$. Define the entropy augmented reward as $r_{\pi}(s,a)=r(s,a)-\gamma\alpha\log\pi(a'|s')$ and rewrite the distributional soft Bellman operator as
\begin{equation}
\nonumber
\mathcal{T}^{\pi}_{\mathcal{D}}Z(s,a) \overset{D}{=}r_{\pi}(s,a)+\gamma Z(s',a'),
\end{equation}
where $r\sim R(\cdot|s,a),s'\sim p,a'\sim\pi$. Then we can directly apply the standard convergence results for policy evaluation of distributional RL \cite{bellemare2017C51}, that is,  $\mathcal{T}^{\pi}_{\mathcal{D}}: \overline{Z}\rightarrow\overline{Z}$ is a $\gamma$-contraction in terms of some measure. Therefore, $\mathcal{T}^{\pi}_{\mathcal{D}}$ has a unique fixed point, which is $Z^{\pi}$, and the sequence $Z^{i}$ will converge to it as $i\rightarrow \infty$, i.e., $\mathcal{Z}^{i}$ will converge to $\mathcal{Z}^{\pi}$ as $i\rightarrow \infty$.
\end{proof}

\begin{lemma}\label{lemma.dspi}
(Soft Policy Improvement) Let $\pi_{\text{new}}$ be the optimal solution of the maximization problem defined in \eqref{eq.policy_imp}. Then $Q^{\pi_{\text{new}}}(s,a) \ge Q^{\pi_{\text{old}}}(s,a)$ for $\forall(s,a) \in \mathcal{S}\times\mathcal{A}$.
\begin{proof}
From \eqref{eq.policy_imp}, one has
\begin{equation}
    \pi_{\text{new}}(\cdot|s) = \arg\max_{\pi}\Exp_{a\sim \pi}[Q^{\pi_{\text{old}}}(s,a)-\alpha\log\pi(a|s)], \quad\forall s \in \mathcal{S},
\end{equation}
then it is obvious that
\begin{equation}
\begin{aligned}
&\Exp_{a\sim \pi_{\text{new}}}[Q^{\pi_{\text{old}}}(s,a)-\alpha\log\pi_{\text{new}}(a|s)]\ge \\
& \qquad \quad \Exp_{a\sim \pi_{\text{old}}}[Q^{\pi_{\text{old}}}(s,a)-\alpha\log\pi_{\text{old}}(a|s)], \quad \forall s \in \mathcal{S}.
\end{aligned}
\end{equation} 
Next, from \eqref{eq.soft_bellman}, it follows that
\begin{equation}
\nonumber
\begin{aligned}
&Q^{\pi_{\text{old}}}(s, a) \\
&= \Exp_{r\sim R(\cdot|s,a)}[r]+\gamma\Exp_{s'\sim p,a'\sim \pi_{\text{old}}}[Q^{\pi_{\text{old}}}(s',a')-\alpha\log\pi_{\text{old}}(a'|s')]\\
&\le \Exp_{r\sim R(\cdot|s,a)}[r]+\gamma\Exp_{s'\sim p,a'\sim \pi_{\text{new}}}[Q^{\pi_{\text{old}}}(s',a')-\alpha\log\pi_{\text{new}}(a'|s')]\\
&\vdots\\
&\le Q^{\pi_{\text{new}}}(s,a),  \quad \forall(s,a)\in\mathcal{S}\times\mathcal{A},
\end{aligned}
\end{equation}
where we have repeatedly expanded $Q^{\pi_{\text{old}}}$ on the right-hand side by applying \eqref{eq.soft_bellman}.
\end{proof}
\end{lemma}

\begin{theorem}\label{theorem.dspi}
(Distributional Soft Policy Iteration). The distributional soft policy iteration, which alternates between distributional soft policy evaluation and soft policy improvement, can converge to a policy $\pi^*$ such that $Q^{\pi^*}(s, a)\ge Q^{\pi}(s, a)$ for $\forall\pi$ and $\forall (s, a)\in\mathcal{S}\times\mathcal{A}$, assuming that $|\mathcal{A}|<\infty$ and reward is bounded.
\end{theorem}
\begin{proof}
Let $\pi_k$ denote the policy at iteration $k$. For $\forall \pi_k$, we can always find its associated $\mathcal{Z}^{\pi_k}$ through distributional soft policy evaluation process follows from Lemma \ref{lemma.dspe}. Therefore, we can obtain $Q^{\pi_k}$ according to \eqref{eq.Q_equal_exp_Z}. By Lemma \ref{lemma.dspi}, the sequence $Q^{\pi_k}(s,a)$ is monotonically increasing for $\forall(s,a)\in\mathcal{S}\times\mathcal{A}$. Since $Q^{\pi}$ is bounded everywhere for $\forall\pi$ (both the reward and policy entropy are bounded), the policy sequence  $\pi_k$ converges to some $\pi^{\dagger}$ as $k\rightarrow\infty$. At convergence, it must follow that
\begin{equation}
\begin{aligned}
&\Exp_{a\sim \pi^{\dagger}}[Q^{\pi^{\dagger}}(s,a)-\alpha\log\pi^{\dagger}(a|s)]\ge \\
&\qquad \quad \Exp_{a\sim \pi}[Q^{\pi^{\dagger}}(s,a)-\alpha\log\pi(a|s)],\quad \forall \pi, \forall s \in \mathcal{S}.
\end{aligned}
\end{equation}
Using the same iterative argument as in Lemma \ref{lemma.dspi}, we have
\begin{equation}
\nonumber
 Q^{\pi^{\dagger}}(s, a) \ge Q^{\pi}(s, a), \quad \forall \pi, \forall(s,a)\in\mathcal{S}\times\mathcal{A}.
\end{equation}
Hence $\pi^{\dagger}$ is optimal, i.e., $\pi^{\dagger}=\pi^*$.
\end{proof}

\section{Derivations}
\subsection{Derivation of the Standard Deviation in Distributional Q-learning}
\label{appen.derivation_std}
Since the random error $\epsilon_Q$ in \eqref{eq.approximate_error} is assumed to be independent of $(s,a)$, $\delta$ in \eqref{eq.delta_definition} can be further expressed as 
\begin{equation}
\nonumber
\begin{aligned}
\delta & = \mathbb{E}_{\epsilon_Q'}\big[\mathbb{E}_{s'}[\max_{a'}Q(s',a')]\big]-\mathbb{E}_{s'}[\max_{a'}\tilde{Q}(s',a')]\\
&=\mathbb{E}_{\epsilon_Q'}\big[\mathbb{E}_{s'}[\max_{a'}Q_{\theta}(s',a')-\max_{a'}\tilde{Q}(s',a')]\big].
\end{aligned}
\end{equation}
Defining $\eta=\mathbb{E}_{s'}\big[\max_{a'}Q_{\theta}(s',a')-\max_{a'}\tilde{Q}(s',a')\big]$, it follows that
\begin{equation}
\nonumber
\delta=\mathbb{E}_{\epsilon_Q'}[\eta].
\end{equation}
From \eqref{eq.phi_appro_DRL}, we linearize the post-update standard deviation around $\psi$ using Taylor’s expansion
\begin{equation}
\nonumber
\begin{aligned}
&\sigma_{\psi_{\rm{new}}}(s,a)\approx \\
&\qquad\sigma_{\psi}(s,a)+\beta\frac{\Delta\sigma^2+(y-Q_{\theta}(s,a))^2}{{\sigma_{\psi}(s,a)}^3}\|\nabla_{\psi}\sigma_{\psi}(s,a)\|_2^2.
\end{aligned}
\end{equation}
Then, in expectation, the post-update standard deviation is
\begin{equation}
\nonumber
\begin{aligned}
&\mathbb{E}_{\epsilon_Q,\epsilon_Q'}[\sigma_{\psi_{\rm{new}}}(s,a)]\approx \sigma_{\psi}(s,a) +\\
&\qquad\qquad \beta\frac{\Delta\sigma^2+\mathbb{E}_{\epsilon_Q,\epsilon_Q'}[(y-Q_{\theta}(s,a))^2]}{{\sigma_{\psi}(s,a)}^3}\|\nabla_{\psi}\sigma_{\psi}(s,a)\|_2^2.
\end{aligned}
\end{equation}
Since $\Exp_{\epsilon_Q}[\epsilon_Q]=0$, the $\mathbb{E}_{\epsilon_Q,\epsilon_Q'}[(y-Q_{\theta}(s,a))^2]$ term can be expanded as 
\begin{equation}
\nonumber
\begin{aligned}
&\mathbb{E}_{\epsilon_Q,\epsilon_Q'}[(y-Q_{\theta}(s,a))^2]\\
&\ = \mathbb{E}_{\epsilon_Q,\epsilon_Q'}\big[(\mathbb{E}[r] + \gamma\mathbb{E}_{s'}[\max_{a'}Q_{\theta}(s',a')]-Q_{\theta}(s,a))^2\big]\\
&\ =\mathbb{E}_{\epsilon_Q,\epsilon_Q'}\big[(\mathbb{E}[r] + \gamma\mathbb{E}_{s'}[\max_{a'}\tilde{Q}(s',a')]+\gamma\eta-\tilde{Q}(s,a)-\epsilon_Q)^2\big]\\
&\ =\mathbb{E}_{\epsilon_Q,\epsilon_Q'}\big[(\tilde{y}-\tilde{Q}(s,a)+\gamma\eta-\epsilon_Q)^2\big]\\
&\ =(\tilde{y}-\tilde{Q}(s,a))^2+\mathbb{E}_{\epsilon_Q,\epsilon_Q'}\big[(\gamma\eta-\epsilon_Q)^2\big]+\\
& \qquad \qquad  \qquad \qquad  \qquad \qquad  \mathbb{E}_{\epsilon_Q,\epsilon_Q'}\big[2(\tilde{y}-\tilde{Q}(s,a))(\gamma\eta-\epsilon_Q)\big]\\
&\ =(\tilde{y}-\tilde{Q}(s,a))^2+\gamma^2\mathbb{E}_{\epsilon_Q'}[\eta^2]+\mathbb{E}_{\epsilon_Q}[{\epsilon_Q}^2]+ \\
& \qquad 2\gamma(\tilde{y}-\tilde{Q}(s,a))\mathbb{E}_{\epsilon_Q'}[\eta]-2(\gamma\mathbb{E}_{\epsilon_Q'}[\eta]+\tilde{y}-\tilde{Q}(s,a))\mathbb{E}_{\epsilon_Q}[\epsilon_Q]\\ 
&\ =(\tilde{y}-\tilde{Q}(s,a))^2+\gamma^2\mathbb{E}_{\epsilon_Q'}[\eta^2]+\mathbb{E}_{\epsilon_Q}[{\epsilon_Q}^2]+2\gamma\delta(\tilde{y}-\tilde{Q}(s,a)).
\end{aligned}
\end{equation}

In an ideal situation, when $\tilde{Q}(s,a)=\tilde{y}$, that is, $\tilde{Q}(s,a)$ has converged after a period of learning, we further have
\begin{equation}
\nonumber
\mathbb{E}_{\epsilon_Q,\epsilon_Q'}[(y-Q_{\theta}(s,a))^2]=\gamma^2\mathbb{E}_{\epsilon_Q'}[\eta^2]+\mathbb{E}_{\epsilon_Q}[{\epsilon_Q}^2].
\end{equation}
Furthermore, since $\mathbb{E}_{\epsilon_Q'}[\eta^2]\ge \mathbb{E}_{\epsilon_Q'}{[\eta]}^2$, we have
\begin{equation}
\nonumber
\begin{aligned}
&\mathbb{E}_{\epsilon_Q,\epsilon_Q'}[\sigma_{\psi_{\rm{new}}}(s,a)]\\
&\approx \sigma_{\psi}(s,a)+\beta\frac{\Delta\sigma^2+\gamma^2\mathbb{E}_{\epsilon_Q'}[\eta^2]+\mathbb{E}_{\epsilon_Q}[{\epsilon_Q}^2]}{{\sigma_{\psi}(s,a)}^3}\|\nabla_{\psi}\sigma_{\psi}(s,a)\|_2^2\\
&\ge\sigma_{\psi}(s,a)+\beta\frac{\Delta\sigma^2+\gamma^2\mathbb{E}_{\epsilon_Q'}{[\eta]}^2+\mathbb{E}_{\epsilon_Q}[{\epsilon_Q}^2]}{{\sigma_{\psi}(s,a)}^3}\|\nabla_{\psi}\sigma_{\psi}(s,a)\|_2^2\\
&=\sigma_{\psi}(s,a)+\beta\frac{\Delta\sigma^2+\gamma^2\delta^2+\mathbb{E}_{\epsilon_Q}[{\epsilon_Q}^2]}{{\sigma_{\psi}(s,a)}^3}\|\nabla_{\psi}\sigma_{\psi}(s,a)\|_2^2.
\end{aligned}
\end{equation}

\subsection{Derivation of the Objective Function for Soft Return Distribution Update }
\label{appen.derivation_object}
From \eqref{eq.distributional_Bellman}, the loss function for soft state-action return distribution under the KL-divergence measurement is
\begin{equation}
\nonumber
\begin{aligned}
&J_{\mathcal{Z}}(\theta)\\
&=  \mathbb{E}_{(s,a)\sim\mathcal{B}}\Big[D_{\rm{KL}}(\mathcal{T}^{\pi_{\phi'}}_{\mathcal{D}}\mathcal{Z}_{\theta'}(\cdot|s,a),\mathcal{Z}_{\theta}(\cdot|s,a))\Big]\\
&=  \mathbb{E}_{(s,a)\sim\mathcal{B}}\Big[\sum_{\mathcal{T}^{\pi_{\phi'}}_{\mathcal{D}}Z(s,a)}\mathcal{P}(\mathcal{T}^{\pi_{\phi'}}_{\mathcal{D}}Z(s,a)|\mathcal{T}^{\pi_{\phi'}}_{\mathcal{D}}\mathcal{Z}_{\theta'}(\cdot|s,a))\\
& \qquad\qquad\qquad\qquad\qquad \log\frac{\mathcal{P}(\mathcal{T}^{\pi_{\phi'}}_{\mathcal{D}}Z(s,a)|\mathcal{T}^{\pi_{\phi'}}_{\mathcal{D}}\mathcal{Z}_{\theta'}(\cdot|s,a))}{\mathcal{P}(\mathcal{T}^{\pi_{\phi'}}_{\mathcal{D}}Z(s,a)|\mathcal{Z}_{\theta}(\cdot|s,a))}\Big]\\
&=  -\mathbb{E}_{(s,a)\sim\mathcal{B}}\Big[\sum_{\mathcal{T}^{\pi_{\phi'}}_{\mathcal{D}}Z(s,a)}\mathcal{P}(\mathcal{T}^{\pi_{\phi'}}_{\mathcal{D}}Z(s,a)|\mathcal{T}^{\pi_{\phi'}}_{\mathcal{D}}\mathcal{Z}_{\theta'}(\cdot|s,a))\\
&\qquad \qquad \qquad \qquad \qquad \log\mathcal{P}(\mathcal{T}^{\pi_{\phi'}}_{\mathcal{D}}Z(s,a)|\mathcal{Z}_{\theta}(\cdot|s,a))\Big]+c\\
&=  -\mathbb{E}_{(s,a)\sim\mathcal{B}}\Big[\mathbb{E}_{\mathcal{T}^{\pi_{\phi'}}_{\mathcal{D}}Z(s,a)\sim\mathcal{T}^{\pi_{\phi'}}_{\mathcal{D}}\mathcal{Z}_{\theta'}(\cdot|s,a)}\\
&\qquad \qquad \qquad \qquad \qquad \log\mathcal{P}(\mathcal{T}^{\pi_{\phi'}}_{\mathcal{D}}Z(s,a)|\mathcal{Z}_{\theta}(\cdot|s,a))\Big]+c\\
&=  -\mathbb{E}_{(s,a)\sim\mathcal{B}}\Big[\Exp_{\substack{(r,s')\sim \mathcal{B},a'\sim\pi_{\phi'},\\ Z(s',a')\sim\mathcal{Z}_{\theta'}(\cdot|s',a')}}\\
&\qquad \qquad \qquad \qquad \qquad\log\mathcal{P}(\mathcal{T}^{\pi_{\phi'}}_{\mathcal{D}}Z(s,a)|\mathcal{Z}_{\theta}(\cdot|s,a))\Big]+c\\
&= -\Exp_{\substack{(s,a,r,s')\sim\mathcal{B},a'\sim\pi_{\phi'},\\Z(s',a')\sim\mathcal{Z}_{\theta'}(\cdot|s',a')}}\Big[\log\mathcal{P}(\mathcal{T}^{\pi_{\phi'}}_{\mathcal{D}}Z(s,a)|\mathcal{Z}_{\theta}(\cdot|s,a))\Big]+c,
\end{aligned}
\end{equation}
where $c$ is an item independent of $\theta$.

\subsection{Probability Density of the Bounded Actions}
\label{appen.pdf_log_pi}
For algorithms with stochastic policy, we use an unbounded Gaussian as the action distribution $\mu$. However, in practice, the action usually needs to be bounded to a finite interval denoted as $[a_{\text{min}},a_{\text{max}}]$, where $a_{\text{min}}\in\mathbb{R}^{{\rm{dim}}(\mathcal{A})}$ and $a_{\text{max}}\in\mathbb{R}^{{\rm{dim}}(\mathcal{A})}$. Let $u\in\mathbb{R}^{{\rm{dim}}(\mathcal{A})}$ denote a random variable sampled from $\mu$. To account for the action constraint, we project $u$ into a desired action by 
\begin{equation}
\nonumber
a=\frac{a_{\text{max}}-a_{\text{min}}}{2}\odot\tanh(u)+\frac{a_{\text{max}}+a_{\text{min}}}{2},
\end{equation}
where $\odot$ represents the Hadamard product and $\tanh$ is applied element-wise. From \cite{Haarnoja2018SAC}, the probability density of $a$ is given by
\begin{equation}
\nonumber
\pi(a|s)=\mu(u|s)\Big|\det\Big(\frac{\text{d}a}{\text{d}u}\Big)\Big|^{-1}.
\end{equation}
The log-likelihood of $\pi(a|s)$ can be expressed as
\begin{equation}
\nonumber
\begin{aligned}
\log\pi(a|s)=&\log\mu(u|s)-\\
&\sum^{{\rm{dim}}(\mathcal{A})}_{i=1}\Big(\log(1-\tanh^2(u_i))+\log\frac{{a_{\text{max}}}_i-{a_{\text{min}}}_i}{2}\Big).
\end{aligned}
\end{equation}

\subsection{Policy Update Gradients Based on the Soft-Action Return}
\label{appen.policy_update}
If $Q_{\theta}(s,a)$ cannot be expressed explicitly through $\theta$, besides \eqref{eq.reparameter_a}, we also need to reparameterize the random return $Z(s,a)$ as 
\begin{equation}
\nonumber
Z(s,a)=g_{\theta}(\xi_Z;s,a).
\end{equation}
In this case, we have 
\begin{equation}
\nonumber
\begin{aligned}
&\nabla_{\phi}J_{\pi}(\phi)=\mathbb{E}_{s\sim \mathcal{B},\xi_Z,\xi_a}\Big[-\alpha\nabla_{\phi}\log(\pi_{\phi}(a|s))+
\\&\qquad(\nabla_ag_{\theta}(\xi_Z;s,a)-\alpha\nabla_a\log(\pi_{\phi}(a|s)))\nabla_{\phi}f_{\phi}(\xi_a;s)\Big].
\end{aligned}
\end{equation}
Besides, the distribution $\mathcal{Z_{\theta}}$ offers a richer set of predictions for learning than its expected value $Q_{\theta}$. Therefore, we can also choose to maximize the $i$th percentile of $\mathcal{Z_{\theta}}$
\begin{equation}
\nonumber
J_{\pi,i}(\phi)=\mathbb{E}_{s\sim \mathcal{B},a\sim\pi_{\phi}}[P_i(\mathcal{Z}_{\theta}(s,a))-\alpha\log(\pi_{\phi}(a|s))],
\end{equation}
where $P_i$ denotes the $i$th percentile. For example, $i$ should be a smaller value for risk-aware policies learning. The gradients of this objective can also be easily approximated using the reparamterization trick. 

\section{Experimental Details}
\label{appen:baseline algorithms}
{\color{black}{
\subsection{Brief Descriptions of Benchmarks}
\label{appen:benchmarks}
The Humanoid-v2 task aims to make a three-dimensional bipedal robot walk forward as fast as possible, without falling over. Its state is described by 376-dimensional information, including the position and velocity of joints, the inertia and velocity at the center of mass, and actuator forces. The action of this task is composed of the torque applied over 17 joints. The reward function is designed to punish the actions that cost a lot of energy or cause mission failure. Similarly, Walker2d-v2 is a two-dimensional bipedal robot which possesses 17-dimensional states and 6-dimensional actions. The Ant-v2 task aims to make a four-legged creature walk forward as fast as possible with a 111-dimensional state vector to describe the position and velocity of each joint. Its action consists of the torque of 8 joints, and the reward is also designed to punish the actions that cost a lot of energy or cause mission failure. Analogously, HalfCheetah-v2 is a two-legged cheetah with 17-dimensional states and 6-dimensional actions. The goal of InvertedDoublePendulum-v2, which is described by an 11-dimensional state vector, is to make two linked poles stand up on a cart as long as possible by applying a force on the cart. See https://github.com/openai/gym/tree/master/gym/envs for all details.}}

\subsection{Double-Q SAC Algorithm}
\label{appen:Double-Q SAC Algorithm}
Suppose the soft Q-value and policy are approximated by parameterized functions $Q_{\theta}(s,a)$ and $\pi_{\phi}(a|s)$ respectively. A pair of soft Q-value functions $(Q_{\theta_1},Q_{\theta_2})$ and policies $(\pi_{\phi_1},\pi_{\phi_2})$ are required in Double-Q SAC, where $\pi_{\phi_1}$ is updated with respect to $Q_{\theta_1}$ and $\pi_{\phi_2}$ with respect to $Q_{\theta_2}$. Given separate target soft Q-value functions $(Q_{\theta_1'},Q_{\theta_2'})$ and policies $(\pi_{\phi_1'},\pi_{\phi_2'})$, the update targets of $Q_{\theta_1}$ and $Q_{\theta_2}$ are calculated as:
\begin{equation}
\nonumber
\begin{aligned}
y_1 = r + \gamma (Q_{\theta_2'}(s',a')-\alpha\log(\pi_{\phi_1'}(a'|s'))), \ a'\sim\pi_{\phi_1'},\\
y_2 = r + \gamma (Q_{\theta_1'}(s',a')-\alpha\log(\pi_{\phi_2'}(a'|s'))), \ a'\sim\pi_{\phi_2'}.
\end{aligned}
\end{equation}
The soft Q-value can be trained by directly minimizing
\begin{equation}
\nonumber
J_{Q}(\theta_i) = \Exp_{(s,a,r,s')\sim\mathcal{B}, a'\sim\pi_{\phi_i'}}\big[(y_i-Q_{\theta_i}(s,a))^2\big], \  {\rm{for}} \ i\in\{1,2\}.
\end{equation}

The policy can be learned by directly maximizing a parameterized variant of the objective function in \eqref{eq.policy_imp}
\begin{equation}
\nonumber
J_{\pi}(\phi_i)=\mathbb{E}_{s\sim\mathcal{B}}\big[\mathbb{E}_{a\sim\pi_{\phi_i}}[Q_{\theta_i}(s,a)-\alpha\log(\pi_{\phi_i}(a|s))]\big].
\end{equation}
The pseudo-code of Double-Q SAC is shown in Algorithm \ref{alg:Double-Q SAC}.

\begin{algorithm}[!htb]
\caption{Double-Q SAC Algorithm}
\label{alg:Double-Q SAC}
\begin{algorithmic}
\STATE Initialize parameters $\theta_1$, $\theta_2$, $\phi_1$, $\phi_2$, and $\alpha$
\STATE Initialize target parameters $\theta_1'\leftarrow\theta_1$, $\theta_2'\leftarrow\theta_2$, $\phi_1'\leftarrow\phi_1$, $\phi_2'\leftarrow\phi_2$
\STATE Initialize learning rate $\beta_{Q}$, $\beta_{\pi}$, $\beta_{\alpha}$ and $\tau$ 
\STATE Initialize iteration index $k=0$

\REPEAT
\STATE Select action $a\sim\pi_{\phi_1}(a|s)$
\STATE Observe reward $r$ and new state $s'$
\STATE Store transition tuple $(s,a,r,s')$ in buffer $\mathcal{B}$
\STATE Sample $N$ transitions $(s,a,r,s')$ from $\mathcal{B}$
\STATE Update soft Q $\theta_i \leftarrow \theta_i - \beta_{Q}\nabla_{\theta_i}J_{Q}(\theta_i)$ for $i\in\{1,2\}$
\IF{$k$ mod $m$}
\STATE Update policy $\phi_i \leftarrow \phi_i + \beta_{\pi}\nabla_{\phi_i}J_{\pi}(\phi_i)$ for $i\in\{1,2\}$
\STATE Adjust temperature $\alpha \leftarrow \alpha - \beta_{\alpha}\nabla_{\alpha} J(\alpha)$
\STATE Update target networks:
\STATE \qquad $\theta_i' \leftarrow  \tau\theta_i+(1-\tau)\theta_i'$ for $i\in\{1,2\}$
\STATE \qquad $\phi_i' \leftarrow  \tau\phi_i+(1-\tau)\phi_i'$ for $i\in\{1,2\}$
\ENDIF
\STATE $k=k+1$
\UNTIL Convergence  
\end{algorithmic}
\end{algorithm}

\subsection{Single-Q SAC Algorithm}
\label{appen:single-Q SAC}
Suppose the soft Q-value and policy are approximated by parameterized functions $Q_{\theta}(s,a)$ and $\pi_{\phi}(a|s)$ respectively. Given separate target soft Q-value function $Q_{\theta'}$ and policy $\pi_{\phi'}$, the update target of $Q_{\theta}$ is calculated as:
\begin{equation}
\nonumber
y = r + \gamma (Q_{\theta'}(s',a')-\alpha\log(\pi_{\phi'}(a'|s'))), \ a'\sim\pi_{\phi'}.
\end{equation}
The soft Q-value can be trained by directly minimizing
\begin{equation}
\nonumber
J_{Q}(\theta) = \Exp_{(s,a,r,s')\sim\mathcal{B}, a'\sim\pi_{\phi'}}\big[(y-Q_{\theta}(s,a))^2\big].
\end{equation}

The policy can be learned by directly maximizing a parameterized variant of the objective function in \eqref{eq.policy_imp}
\begin{equation}
\nonumber
J_{\pi}(\phi)=\mathbb{E}_{s\sim\mathcal{B}}\big[\mathbb{E}_{a\sim\pi_{\phi}}[Q_{\theta}(s,a)-\alpha\log(\pi_{\phi}(a|s))]\big].
\end{equation}
The pseudo-code of Single-Q SAC is shown in Algorithm \ref{alg:Single-Q SAC}.

\begin{algorithm}[!htb]
\caption{Single-Q SAC Algorithm}
\label{alg:Single-Q SAC}
\begin{algorithmic}
\STATE Initialize parameters $\theta$, $\phi$ and $\alpha$
\STATE Initialize target parameters $\theta'\leftarrow\theta$,  $\phi'\leftarrow\phi$
\STATE Initialize learning rate $\beta_{Q}$, $\beta_{\pi}$, $\beta_{\alpha}$ and $\tau$ 
\STATE Initialize iteration index $k=0$

\REPEAT
\STATE Select action $a\sim\pi_{\phi}(a|s)$
\STATE Observe reward $r$ and new state $s'$
\STATE Store transition tuple $(s,a,r,s')$ in buffer $\mathcal{B}$
\STATE Sample $N$ transitions $(s,a,r,s')$ from $\mathcal{B}$
\STATE Update soft Q-function $\theta \leftarrow \theta - \beta_{Q}\nabla_{\theta}J_{Q}(\theta)$
\IF{$k$ mod $m$}
\STATE Update policy $\phi \leftarrow \phi + \beta_{\pi}\nabla_{\phi}J_{\pi}(\phi)$ 
\STATE Adjust temperature $\alpha \leftarrow \alpha - \beta_{\alpha}\nabla_{\alpha} J(\alpha)$
\STATE Update target networks:
\STATE \qquad $\theta' \leftarrow  \tau\theta+(1-\tau)\theta'$, $\phi' \leftarrow  \tau\phi+(1-\tau)\phi'$ 
\ENDIF
\STATE $k=k+1$
\UNTIL Convergence  
\end{algorithmic}
\end{algorithm}

\subsection{TD4 Algorithm}
\label{appen:TD4}
Consider a parameterized state-action return distribution function $\mathcal{Z}_{\theta}(\cdot|s,a)$ and a deterministic policy $\pi_{\phi}(s)$, where $\theta$ and $\phi$ are parameters. The target networks $\mathcal{Z}_{\theta'}(\cdot|s,a)$ and $\pi_{\phi'}(s)$ are used to stabilize learning. 
The return distribution can be trained to minimize 
\begin{equation}
\nonumber
J_{\mathcal{Z}}(\theta)=-\Exp_{\substack{(s,a,r,s')\sim \mathcal{B},a'\sim\pi_{\phi'},\\ Z(s',a')\sim \mathcal{Z}_{\theta'}(\cdot|s',a')}}\Big[\log\mathcal{P}(\mathcal{T}^{\pi_{\phi'}}_{\mathcal{D}}Z(s,a)|\mathcal{Z}_{\theta}(\cdot|s,a))\Big],
\end{equation}
where
\begin{equation}
\nonumber
\mathcal{T}^{\pi}_{\mathcal{D}}Z(s,a) \overset{D}{=}r(s,a)+\gamma Z(s',a')
\end{equation}
and 
\begin{equation}
\nonumber
 a'=\pi_{\phi'}(s')+\epsilon, \ \epsilon\sim{\rm{clip}}(\mathcal{N}(0,\sigma^2),-c,c).
\end{equation}

The calculation of  $\nabla_{\theta}J_{\mathcal{Z}}(\theta)$ is similar to DSAC. The policy can be learned by directly maximizing the expected return 
\begin{equation}
\nonumber
J_{\pi}(\phi)=\mathbb{E}_{s\sim\mathcal{B}}\big[Q_{\theta}(s,\pi_{\phi}(s))\big].
\end{equation}
The pseudo-code is shown in Algorithm \ref{alg:TD4}.

\begin{algorithm}[!htb]
\caption{TD4 Algorithm}
\label{alg:TD4}
\begin{algorithmic}
\STATE Initialize parameters $\theta$, $\phi$ and $\alpha$
\STATE Initialize target parameters $\theta'\leftarrow\theta$, $\phi'\leftarrow\phi$
\STATE Initialize learning rate $\beta_{\mathcal{Z}}$, $\beta_{\pi}$, $\beta_{\alpha}$ and $\tau$ 
\STATE Initialize iteration index $k=0$
\REPEAT
\STATE Select action with exploration noise $a=\pi_{\phi}(s)+\epsilon$, $\epsilon\sim\mathcal{N}(0,\hat{\sigma}^2)$
\STATE Observe reward $r$ and new state $s'$
\STATE Store transition tuple $(s,a,r,s')$ in buffer $\mathcal{B}$
\STATE Sample $N$ transitions $(s,a,r,s')$ from $\mathcal{B}$
\STATE Calculate action for target policy smoothing $a'=\pi_{\phi'}(s')+\epsilon$, $\epsilon\sim{\rm{clip}}(\mathcal{N}(0,\sigma^2),-c,c)$
\STATE Update return distribution $\theta \leftarrow \theta - \beta_{\mathcal{Z}}\nabla_{\theta}J_{\mathcal{Z}}(\theta)$
\IF{$k$ mod $m$}
\STATE Update policy $\phi \leftarrow \phi + \beta_{\pi}\nabla_{\phi} J_{\phi}(\phi)$
\STATE Update target networks:
\STATE \qquad $\theta' \leftarrow  \tau\theta+(1-\tau)\theta'$, $\phi' \leftarrow  \tau\phi+(1-\tau)\phi'$
\ENDIF
\STATE $k=k+1$
\UNTIL Convergence  
\end{algorithmic}
\end{algorithm}

\subsection{Hyperparameters}
\label{appen.hyper}
Table \ref{table.hyper} lists the hyperparameters of all algorithms.
\begin{table}[!htp]
\captionsetup{justification=centering,labelsep=newline,font=small}
\captionsetup{justification=centering,labelsep=newline,font={small,sc}}
\caption{Detailed hyperparameters.}
\label{table.hyper}
\begin{tabular}{lc}
\toprule
Hyperparameters & Value \\
\hline
\emph{Shared} & \\
\quad Optimizer &  Adam ($\beta_{1}=0.9, \beta_{2}=0.999$)\\
\quad Number of hidden layers & 5\\
\quad Number of hidden units per layer & 256\\
\quad Nonlinearity of hidden layer& GELU\\
\quad Replay buffer size & $5\times10^5$\\
\quad Batch size & 256\\
\quad Actor learning rate & $\cos$ anneal $5{\rm{e-}}5\rightarrow1{\rm{e-}}6 $\\
\quad Critic learning rate & $\cos$ anneal $8{\rm{e-}}5\rightarrow1{\rm{e-}}6 $\\
\quad Discount factor ($\gamma$) & 0.99\\
\quad Update interval ($m$)& 2\\
\quad Target smoothing coefficient ($\tau$) & 0.001\\
\quad Reward scale & 0.2\\
\quad Number of actor processes &6\\
\quad Number of learner processes &4\\
\quad Number of buffer processes &3\\ 
\hline
\emph{Stochastic policy} &\\ 
\quad  Learning rate of $\alpha$ & $\cos$ anneal $5{\rm{e-}}5\rightarrow1{\rm{e-}}6 $ \\
\quad  Expected entropy ($\overline{\mathcal{H}}$) &  $\overline{\mathcal{H}}=-{\rm{dim}}(\mathcal{A})$ \\
\hline
\emph{Deterministic policy} &\\ 
\quad Exploration noise&  $\epsilon \sim \mathcal{N}(0,0.1^2)$\\
\hline
\emph{Distributional value function} &\\ 
\quad Bounds of variance &  $\sigma_{\rm{min}}=1$\\
\quad Clipping boundary &  $b=10$ \\
\hline
\emph{TD4,TD3} &\\ 
\quad Policy smoothing noise &  $\epsilon\sim{\rm{clip}}(\mathcal{N}(0,0.2^2),-0.5,0.5)$\\
\bottomrule
\end{tabular}
\end{table}

% use section* for acknowledgment

\section*{Acknowledgment}
We would like to acknowledge Dongjie Yu for his valuable suggestions. \textcolor{black}{The authors are grateful to the Editor-in-Chief, the Associate Editor, and anonymous reviewers for their valuable comments.}

% if have a single appendix:
%\appendix[Proof of the Zonklar Equations]
% or
%\appendix  % for no appendix heading
% do not use \section anymore after \appendix, only \section*
% is possibly needed

% use appendices with more than one appendix
% then use \section to start each appendix
% you must declare a \section before using any
% \subsection or using \label (\appendices by itself
% starts a section numbered zero.)
%

% Can use something like this to put references on a page
% by themselves when using endfloat and the captionsoff option.
\ifCLASSOPTIONcaptionsoff
  \newpage
\fi

% trigger a \newpage just before the given reference
% number - used to balance the columns on the last page
% adjust value as needed - may need to be readjusted if
% the document is modified later
%\IEEEtriggeratref{8}
% The "triggered" command can be changed if desired:
%\IEEEtriggercmd{\enlargethispage{-5in}}

% references section

% can use a bibliography generated by BibTeX as a .bbl file
% BibTeX documentation can be easily obtained at:
% http://mirror.ctan.org/biblio/bibtex/contrib/doc/
% The IEEEtran BibTeX style support page is at:
% http://www.michaelshell.org/tex/ieeetran/bibtex/
%\bibliographystyle{IEEEtran}
% argument is your BibTeX string definitions and bibliography database(s)
%\bibliography{IEEEabrv,../bib/paper}
%
% <OR> manually copy in the resultant .bbl file
% set second argument of \begin to the number of references
% (used to reserve space for the reference number labels box)
\bibliographystyle{ieeetr}
\bibliography{reference}
%\vskip -9\baselineskip 
\begin{IEEEbiography}[{\includegraphics[width=1in,height=1.25in,clip,keepaspectratio]{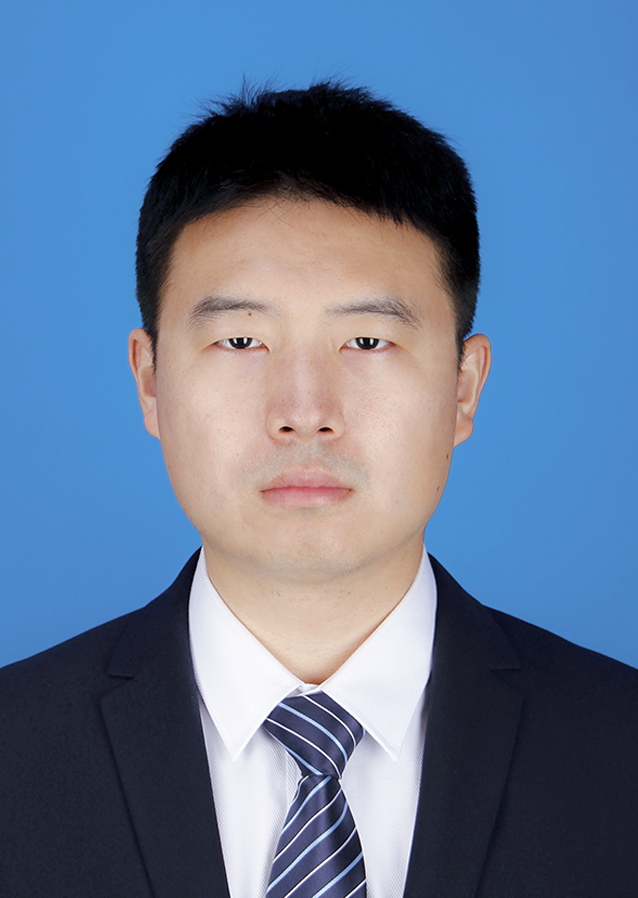}}]{Jingliang Duan}
received the B.S. degree from the College of Automotive Engineering, Jilin University, Changchun, China, in 2015. He studied as a visiting student researcher in Department of Mechanical Engineering, University of California, Berkeley, USA, in 2019. He received his Ph.D. degree in the School of Vehicle and Mobility, Tsinghua University, Beijing, China, in 2021. His research interests include decision and control of autonomous vehicle, reinforcement learning and adaptive dynamic programming, and driver behaviour analysis.
\end{IEEEbiography}
\vskip -2\baselineskip plus -1fil
\begin{IEEEbiography}[{\includegraphics[width=1in,height=1.25in,clip,keepaspectratio]{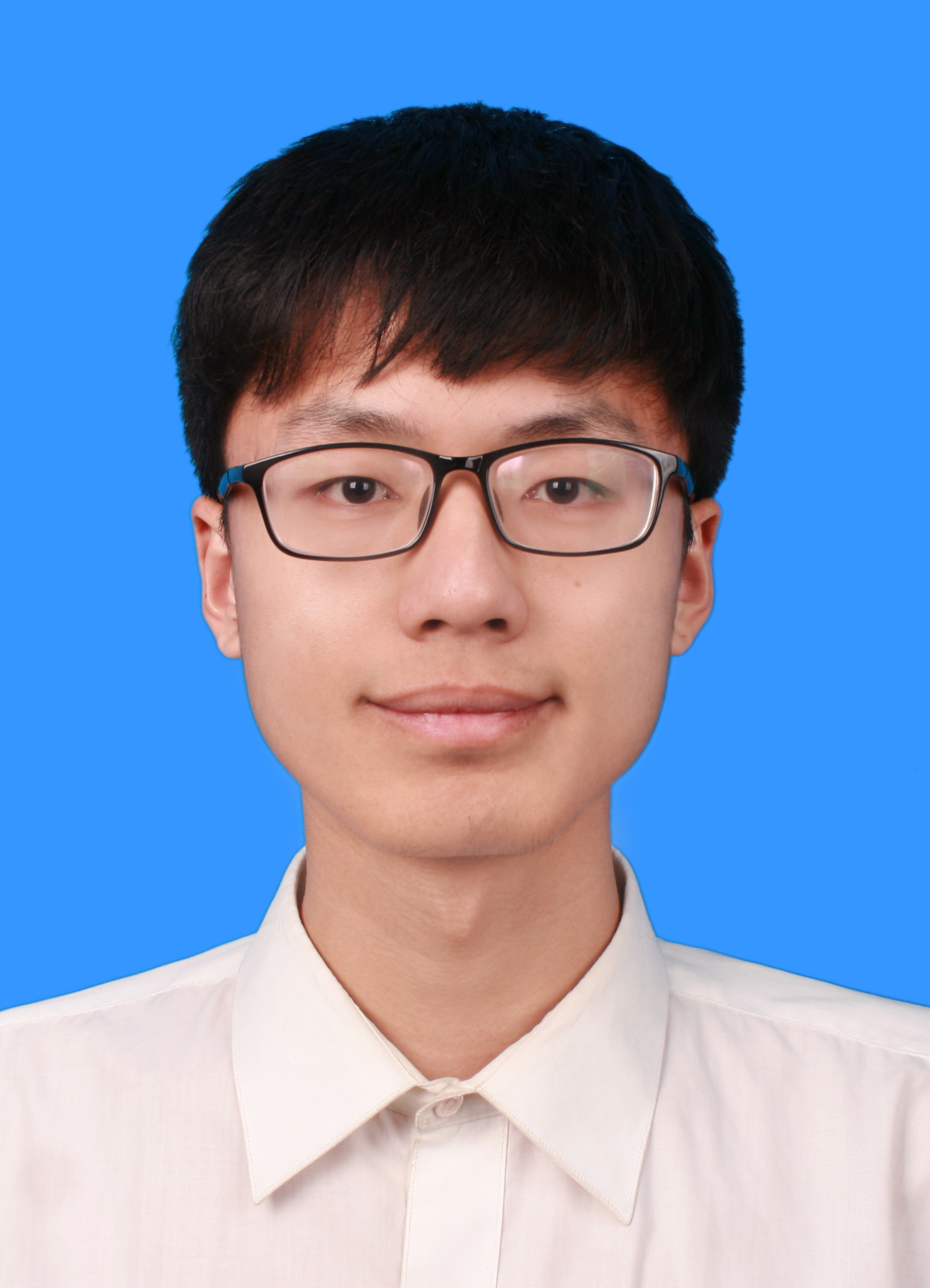}}]{Yang Guan}
received the B.S. degree from school of mechanical engineering, Beijing institute of technology, Beijing, China, in 2017. He is pursuing his Ph.D. degree in the School of Vehicle and Mobility, Tsinghua University, Beijing, China. His research interests include decision-making of autonomous vehicle, and reinforcement learning.
\end{IEEEbiography}
\vskip -2\baselineskip plus -1fil
\begin{IEEEbiography}[{\includegraphics[width=1in,height=1.25in,clip,keepaspectratio]{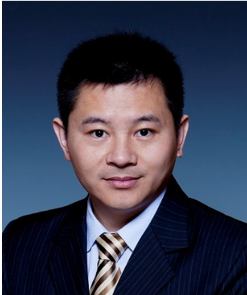}}]{Shengbo Eben Li}
(SM'16) received the M.S. and Ph.D. degrees from Tsinghua University in 2006 and 2009. He worked at Stanford University, University of Michigan, and University of California, Berkeley. He is currently a tenured professor at Tsinghua University. His active research interests include intelligent vehicles and driver assistance, reinforcement learning and distributed control, optimal control and estimation, etc.

He is the author of over 100 journal/conference papers, and the co-inventor of over 20 Chinese patents. He was the recipient of Best Paper Award in 2014 IEEE ITS Symposium, Best Paper Award in 14th ITS Asia Pacific Forum, National Award for Technological Invention in China (2013), Excellent Young Scholar of NSF China (2016), Young Professorship of Changjiang Scholar Program (2016). He is now the IEEE senior member and serves as associated editor of IEEE ITSM and IEEE Trans. ITS, etc.
\end{IEEEbiography}
\vskip -2\baselineskip plus -1fil
\begin{IEEEbiography}[{\includegraphics[width=1in,height=1.25in,clip,keepaspectratio]{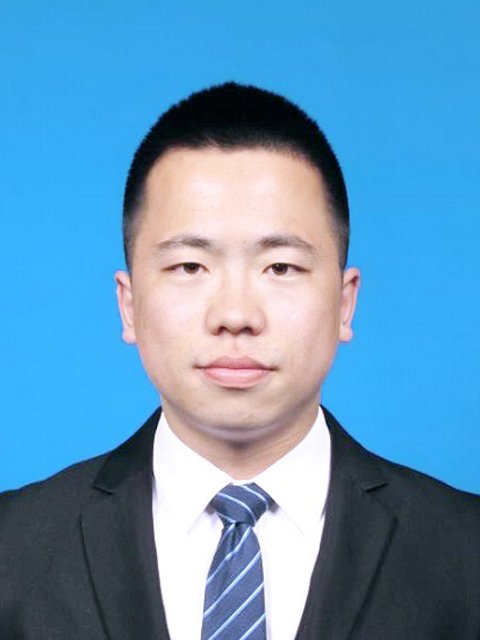}}]{Yangang Ren}
received the B.S. degree from the Department of Automotive Engineering, Tsinghua University, Beijing, China, in 2018. He is currently pursuing his Ph.D. degree in the School of Vehicle and Mobility, Tsinghua University, Beijing, China. His research interests include decision and control of autonomous driving, reinforcement learning, and adversarial learning.
\end{IEEEbiography}
\vskip -2\baselineskip plus -1fil
\begin{IEEEbiography}[{\includegraphics[width=1in,height=1.25in,clip,keepaspectratio]{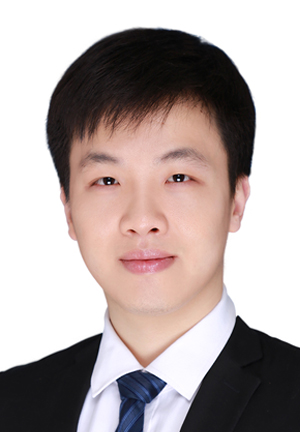}}]{Qi Sun}
received his Ph.D. degree in Automotive Engineering from Ecole Centrale de Lille, France, in 2017. He did scientific research and completed his Ph.D. dissertation in CRIStAL Research Center at Ecole Centrale de Lille, France, between 2013 and 2016. He is currently a Postdoctor at the State Key Laboratory of Automotive Safety and Energy and at the School  of  Vehicle  and  Mobility, Tsinghua University, Beijing, China. His active research interests include intelligent vehicles, automatic driving technology, distributed control and optimal control.
\end{IEEEbiography}
\vskip -2\baselineskip plus -1fil
\begin{IEEEbiography}[{\includegraphics[width=1in,height=1.25in,clip,keepaspectratio]{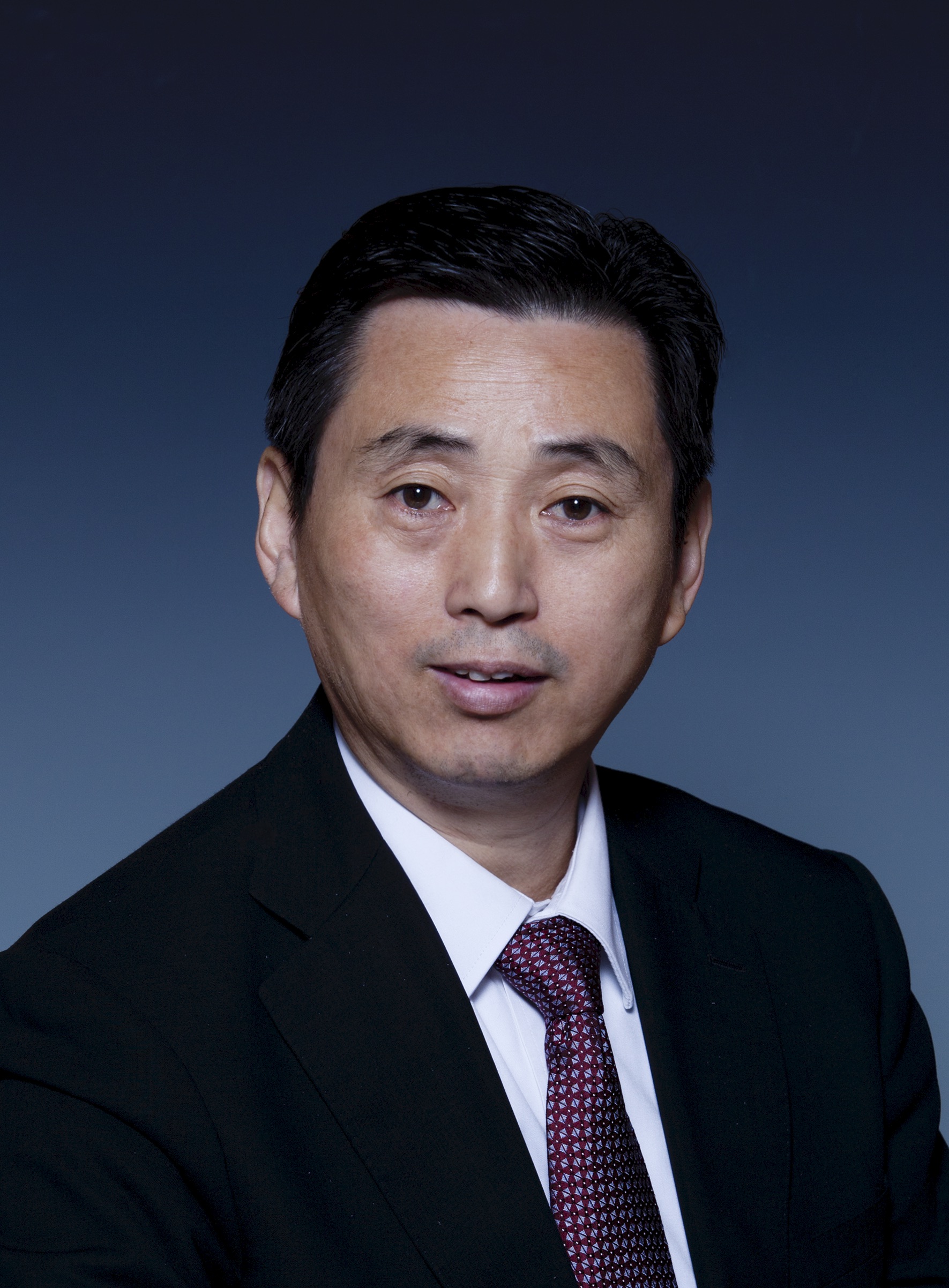}}]{Bo Cheng}
received the B.S. and M.S. degrees in automotive engineering from Tsinghua University, Beijing, China, in 1985 and 1988, respectively, and the Ph.D. degree in mechanical engineering from the University of Tokyo, Tokyo, Japan, in 1998. He is currently a Professor with School of Vehicle and Mobility, Tsinghua University, and the Dean of Tsinghua University–Suzhou Automotive Research Institute. He is the author of more than 100 peer-reviewed journal/conference papers and the co-inventor of 40 patents. His active research interests include autonomous vehicles, driver-assistance systems, active safety, and vehicular ergonomics, among others. Dr. Cheng is also the Chairman of the Academic Board of SAE-Beijing, a member of the Council of the Chinese Ergonomics Society, and a Committee Member of National 863 Plan, among others.
\end{IEEEbiography}

% insert where needed to balance the two columns on the last page with
% biographies
%\newpage

% You can push biographies down or up by placing
% a \vfill before or after them. The appropriate
% use of \vfill depends on what kind of text is
% on the last page and whether or not the columns
% are being equalized.

%\vfill

% Can be used to pull up biographies so that the bottom of the last one
% is flush with the other column.
%\enlargethispage{-5in}

% that's all folks

\end{document}